\documentclass[format=acmsmall, review=false, authorversion=true, screen=true]{acmart}

\pdfoutput=1

\usepackage[ruled,linesnumbered,vlined,norelsize]{algorithm2e}
\usepackage[noend]{algpseudocode}

\usepackage{multirow}
\usepackage{booktabs}
\usepackage{macros}
\usepackage{mathtools}
\usepackage{subfigure}
\usepackage{tikz}
\usepackage{pgfplots,pgfplotstable}
\pgfplotsset{compat=1.13}
\usepackage[unboxed]{cwpuzzle}
\usepackage{etoolbox}

\usetikzlibrary{positioning}
\usetikzlibrary{shapes,arrows,calc,matrix}

\newcommand{\JJoin}{\textsc{AddEq}\xspace}
\newcommand{\UnJoin}{\textsc{DelEq}\xspace}
\newcommand{\Const}{\textsc{DelVar}\xspace}
\newcommand{\UnConst}{\textsc{AddVar}\xspace}
\newcommand{\AddRel}{\textsc{AddConstr}\xspace}
\newcommand{\DelRel}{\textsc{DelConstr}\xspace}

\newcommand{\stdhspace}{\hspace{5em}}

\DeclarePairedDelimiter{\abs}{\lvert}{\rvert}

\makeatletter
\patchcmd{\@algocf@start}% <cmd>
 {-1.5em}% <search>
  {0pt}% <replace>
  {}{}% <success><failure>
\makeatother

\AtBeginDocument{%
  \providecommand\BibTeX{{%
    \normalfont B\kern-0.5em{\scshape i\kern-0.25em b}\kern-0.8em\TeX}}}

\begin{document}
\setcopyright{none}
\settopmatter{printacmref=false,printccs=false,printfolios=true}

\title{Incremental Updates of Generalized Hypertree Decompositions}

\author{Georg Gottlob}
\orcid{0000-0002-2353-5230}
\author{Matthias Lanzinger}
\orcid{0000-0002-7601-3727}
\affiliation{%
  \institution{University of Oxford}
  \city{Oxford}
  \country{UK}
}
\email{georg.gottlob@cs.ox.ac.uk}
\email{matthias.lanzinger@cs.ox.ac.uk}

\author{Davide Mario Longo}
\orcid{0000-0003-4018-4994}
\author{Cem Okulmus}
\orcid{0000-0002-7742-0439}

\affiliation{%
  \institution{TU Wien}
  \city{Vienna}
  \country{Austria
}}

\email{dlongo@dbai.tuwien.ac.at}
\email{cokulmus@dbai.tuwien.ac.at}

\renewcommand{\shortauthors}{Gottlob, Lanzinger, Longo, and Okulmus}

\begin{abstract}
    Structural decomposition methods, such as generalized hypertree decompositions, have been successfully used for solving constraint satisfaction problems (CSPs). As decompositions can be reused to solve CSPs with the same constraint scopes, investing resources in computing good decompositions is beneficial, even though the computation itself is hard.
    Unfortunately, current methods need to compute a completely new decomposition even if the scopes change only slightly.
    In this paper, we make the first steps toward solving the problem of updating the decomposition of a CSP $P$ so that it becomes a valid decomposition of a new CSP $P'$ produced by some modification of $P$.
    Even though the problem is hard in theory, we propose and implement a framework for effectively updating GHDs.
    The experimental evaluation of our algorithm strongly suggests practical applicability.
\end{abstract}




\maketitle

\section{Introduction}

Constraint satisfaction problems (CSPs) are fundamental in modeling many problems of Artificial Intelligence and other areas of Computer Science.
While satisfiability checking of  CSPs is generally \np-hard~\cite{DBLP:conf/stoc/Schaefer78}, the structure of constraints plays a crucial role in their resolution.
This can be represented by a \emph{hypergraph} $H=(V(H), E(H))$, which consists of a set of vertices $V(H)$ and a set of edges $E(H)$ with $E(H) \subseteq 2^{V(H)}$.
Intuitively, the vertices of $H$ correspond to variables and the edges of $H$ group together variables appearing in the same constraint.
It is well known that checking if a CSP is satisfiable is tractable for all CSPs that have an underlying acyclic hypergraph~\cite{DBLP:conf/vldb/Yannakakis81}.

Larger islands of tractability have been discovered by generalizing the concept of hypergraph acyclicity.
In this direction, \emph{hypergraph decompositions} and their associated \emph{width} proved to be essential concepts~\cite{DBLP:journals/ai/GottlobLS00}.
In this work we focus on \emph{generalized hypertree decompositions (GHDs)} and \emph{generalized hypertree width ($\ghw$)}~\cite{DBLP:journals/jcss/GottlobLS02}.
The $\ghw$ of a hypergraph $H$ intuitively measures its degree of cyclicity and a GHD of $H$ can be used to solve the related CSP (definitions in Section~\ref{sec:preliminaries}).
Acyclic hypergraphs have $\ghw = 1$. Furthermore, CSPs with bounded $\ghw$ can be solved in polynomial time~\cite{DBLP:journals/jcss/GottlobLS02}.

Computing a GHD of width $\leq k$ is \np-hard for any $k \geq 2$~\cite{DBLP:conf/pods/FischlGP18,DBLP:journals/jacm/GottlobMS09},
but certain properties of hypergraphs, like bounded intersection width and bounded multi-intersection width, make the problem tractable~\cite{DBLP:conf/pods/FischlGP18,DBLP:journals/jacm/GottlobLPR21}.
This set off a quest for efficient algorithms, which led to implementations based on different principles.
The aforementioned intersection widths were exploited in~\cite{10.1145/3440015,DBLP:conf/ijcai/GottlobOP20} for the implementation of several sequential and parallel algorithms for computing GHDs.
Alternative characterizations of width have been used for even more general forms of decompositions in~\cite{DBLP:conf/cp/FichteHLS18,DBLP:journals/jea/KorhonenBJ19,DBLP:conf/alenex/SchidlerS20}.
Furthermore, the fixed-parameter tractability of the problem has been explored extensively in theory and practice but is outside of the scope of this paper~\cite{DBLP:journals/jacm/Grohe07,DBLP:journals/jacm/Marx13,DBLP:conf/pods/KhamisNR16,DBLP:conf/ijcai/ChenGLP20}.

Decompositions have also been employed in commercial systems and research prototypes, both for CSPs and query answering in databases~\cite{DBLP:conf/sigmod/AbergerTOR16,DBLP:journals/aicom/AmrounHA16,DBLP:conf/sigmod/ArefCGKOPVW15,DBLP:journals/jetai/HabbasAS15,DBLP:conf/aiia/LalouHA09}.
In particular, in~\cite{DBLP:conf/sigmod/AbergerTOR16} GHDs of low width significantly speed up query answering.
It is thus worth investing resources in computing a GHD of low width.
On the other hand, this is a hard task and we want to avoid the computation of a new GHD, whenever possible.
Unfortunately, a new decomposition must be computed even if the CSP slightly changes, resulting in a loss on the investment.
This is the case in the setting of \emph{incremental constraint satisfaction}, where constraint solvers handle mutable sets of variables~\cite{DBLP:conf/ijcai/Seidel81} or constraints~\cite{DBLP:journals/cacm/Freeman-BensonMB90}, as well as for other scenarios.

Consider a user modeling a problem in the context of interactive problem solving.
In this case, the user interactively models a problem and needs prompt feedback on the effect of her modifications on the resolution process.
While investigating alternatives, information about the impact on the decomposition is shown, i.e., an estimation of the computational effort of solving the problem.
Similarly, a \emph{compositional modeling problem} consists in synthesizing the most appropriate model of a physical system for a given analytical query~\cite{DBLP:journals/ai/FalkenhainerF91}.
The construction of the ``best'' model passes through several phases in which the model is iteratively refined by modifying constraints.
Here, support during the modeling process is needed.

In the rest of the paper, we will use the crossword puzzles in Figure~\ref{fig:crossword} for our examples.
Given a puzzle, we want to fill every contiguous horizontal or vertical line of white cells with words from a certain set.
The puzzles are CSPs in which each cell is a variable and there is a constraint over the white cells belonging to the same line.
Suppose we want to solve the puzzle $P$, and then the slightly modified puzzle $P'$ with the help of GHDs.
Although the two puzzles resemble each other, their resolution requires different decompositions.
In particular, even though the hypergraphs of $P$ and $P'$ share a significant part, we have to compute a new GHD to solve $P'$.
Intuitively, it should be possible to obtain a GHD for $P'$ by slightly modifying the already-computed GHD of $P$.
Thus, the question arises naturally: can we reuse the first GHD and adjust only the parts affected by the modification?

\begin{figure}[t]
    \centering
    \subfigure[Puzzle $P$.]{\begin{tikzpicture}[scale=1.2]
	\draw[step=0.5cm] (0,0) grid (2.5,1.5);
	\node at (0.25,1.25)	{$a$};
	\node at (0.75,1.25)	{$b$};
	\node at (1.25,1.25)	{$c$};
	\node at (0.25,0.75)	{$d$};
	\node at (1.25,0.75)	{$e$};
	\node at (0.25,0.25)	{$f$};
	\node at (0.75,+0.25)	{$g$};
	%\node at (1.25,+0.25)	{$h$};
	\node at (1.75,0.75)	{$i$};
	\node at (2.25,1.25)	{$j$};
	\node at (2.25,0.75)	{$k$};
	\node at (2.25,0.25)	{$l$};
	\fill [black] (0.5,0.5) rectangle (1,1);
	\fill [black] (1,0) rectangle (1.5,0.5);
	\fill [black] (1.5,0) rectangle (2,0.5);
	\fill [black] (1.5,1) rectangle (2,1.5);
\end{tikzpicture}}\stdhspace%
    \subfigure[Puzzle $P'$.]{
    	\begin{tikzpicture}[scale=1.2]
	\draw[step=0.5cm] (0,0) grid (2.5,1.5);
	\node at (0.25,1.25)	{$a$};
	\node at (0.75,1.25)	{$b$};
	\node at (1.25,1.25)	{$c$};
	\node at (0.25,0.75)	{$d$};
	\node at (1.25,0.75)	{$e$};
	\node at (0.25,0.25)	{$f$};
	\node at (0.75,+0.25)	{$g$};
	\node at (1.25,+0.25)	{$h$};
	\node at (1.75,0.75)	{$i$};
	\node at (2.25,1.25)	{$j$};
	\node at (2.25,0.75)	{$k$};
	\node at (2.25,0.25)	{$l$};
	\fill [black] (0.5,0.5) rectangle (1,1);
	%\fill [black] (1,0) rectangle (1.5,0.5);
	\fill [black] (1.5,0) rectangle (2,0.5);
	\fill [black] (1.5,1) rectangle (2,1.5);
\end{tikzpicture}
    	\label{fig:cw-cyclic}
    }
    \caption{Two similar crossword puzzles $P$ and $P'$.
    	Given a set of words $W$, we want to fill every contiguous horizontal or vertical line of white cells with words from $W$.
    	If two lines intersect, the words assigned to these lines must intersect in the right positions.
    }
    \label{fig:crossword}
\end{figure}
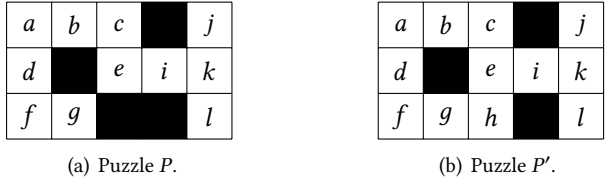

We investigate the problem of incrementally updating GHDs upon constraint modifications.
To this aim we study the behavior of GHDs when the CSP is modified and analyze the problem of computing a GHD of the modified CSP when a GHD of the original CSP is already available.
We also propose a set of typical and natural constraint modifications.
In this work we focus on elementary modifications of CSPs, i.e., changes such as binding a variable to a constant, the introduction of a new constraint, and enforcing equality between a set of variables.
For each modification, we also include its dual function.
Moreover, we briefly discuss how to extend our framework when we are confronted with sequences of updates.
In particular, our contributions are as follows:
\begin{itemize}
    \item We propose a framework for constraint modifications and describe their effect on the hypergraph as well.
    Moreover, we present the novel \minupdateProb{} problem.
    \item We resolve the complexity of \minupdateProb{} for a set of relevant elementary modifications.
    For most of these, the problem turns out to not be solvable in polynomial time under standard assumptions, and is therefore effectively just as difficult as computing a GHD de novo.
    \item Despite its complexity, the problem still offers room for practical solutions.
    To this end, we provide a general framework of \emph{mutable subtrees} for updating GHDs.
    This is devised to work under any kind of modification, thus making our approach universally applicable.
    \item We emphasize the practical applicability of our method by providing an implementation strategy
    that can be used for any top-down algorithm for computing GHDs.
    In this way, we can take advantage of the state-of-the-art solutions in the field.
    \item We extensively compare our method with classical algorithms. 
    Given a modification, we use our algorithm to update an existing decomposition, 
    and a reference classical algorithm to compute a GHD for the new hypergraph afresh.
    Results show that we significantly outperform classical methods for most classes of modifications.
    In particular, we achieve mean speed-ups between 6 and 50 over the reference algorithm.
\end{itemize}

The paper proceeds as follows.
In Section~\ref{sec:preliminaries} we introduce basic concepts.
In Section~\ref{sec:update-problem} we formally define the \minupdateProb{} problem and study its complexity.
In Section~\ref{sec:update-practice} we describe the framework of \magic{} subtrees and an actual implementation strategy for any GHD top-down algorithm.
In Section~\ref{sec:experiments} we present the results of the experimental evaluation of our methods by comparing it to classical algorithms for computing GHDs.
Finally, in Section~\ref{sec:conclusion} we draw conclusions and point to relevant questions for future work.

\section{Preliminaries}
\label{sec:preliminaries}

\subsection{CSPs and Hypergraphs}
A \emph{constraint satisfaction problem (CSP)} is a triple $\langle V,D,C_t \rangle$, where $V$ is a set of variables, $D$ is a set of values, and $C_t$ is a set of constraints.
A constraint $(s_i, r_i) \in C_t$ consists of a tuple of variables $s_i$ and a constraint relation $r_i$ containing valid combinations of values for the variables $s_i$.
A solution is a mapping from $V$ to $D$, s.t. for each $(s_i, r_i) \in C_t$ the variables $s_i$ are mapped to a legal combination of values in $r_i$.

A \emph{hypergraph} $H = (V(H), E(H))$ is a pair consisting of a set of vertices $V(H)$ and a set of non-empty (hyper)edges $E(H) \subseteq 2^{V(H)}$.
We assume w.l.o.g.\  that there are no isolated vertices, i.e., for each $v \in V(H)$, there is at least one edge $e \in E(H)$ such that $v \in e$.
We will often use $H$ to denote the set of edges $E(H)$.
A \emph{subhypergraph} $H'$ of $H$ is then simply a subset of (the edges of) $H$.
Given  $U \subseteq V(H)$ the \emph{induced subhypergraph of $H$ w.r.t. $U$} is the hypergraph $H[U]$ s.t. $V(H[U]) = U$ and $E(H[U]) = \{ e \cap U \mid e \in E(H) \}$.

Let $P = \langle V,D,C_t \rangle$ be a CSP, the hypergraph $H_P$ of $P$ is defined with $V(H_P) = V$ and $E(H_P) = \{s_i \mid (s_i, r_i) \in C_t\}$.

\begin{example}
	A crossword puzzle like the ones in Figure~\ref{fig:crossword} can be represented as a CSP.
	Each cell of the puzzle is a variable in $V$ and, for simplicity, the domain $D$ is the set of letters of the alphabet.
	Given a relation of words $W$ with characters in $D$, the set $C_t$ of constraints contains a constraint $c_i$ for each contiguous horizontal or vertical line of white cells that can be filled with appropriate words in $W$.
	For instance, consider the puzzle $P'$ of Figure~\ref{fig:cw-cyclic}.
	The constraint $c_1$ defined over the variables $w_1 = \langle a,b,c \rangle$ can take values in $r_1 = W$.
	If $W = \mathit{\{\langle d,o,g \rangle, \langle g,o,d \rangle, \langle o,d,d \rangle\}}$, then the assignment $ \langle a,b,c,d,e,f,g,h,i,j,k,l \rangle = \langle \mathrm{g,o,d,o,o,d,o,g,d,o,d,d} \rangle$ is a solution.
	
	The hypergraphs $H_P,H_{P'}$ underlying the puzzles $P,P'$ are shown in Figure~\ref{fig:hypergraphs}.
	The set of vertices of each hypergraph is the set of variables of the corresponding CSP, while the sets of edges match the set of constraint scopes of the related CSP.
\end{example}

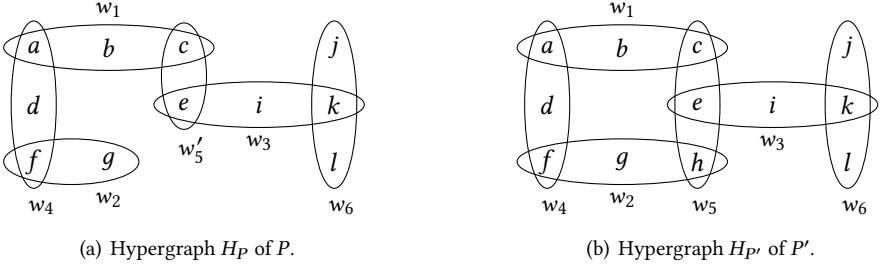
\begin{figure}[t]
	\centering
	\subfigure[Hypergraph $H_P$ of $P$.]{
		\begin{tikzpicture}[scale=1]
	\node	(a)	at	(0,1.5)	{$a$};
	\node	(b)	at	(1,1.5)	{$b$};
	\node	(c)	at	(2,1.5)	{$c$};
	\node	(d)	at	(0,.75)	{$d$};
	\node	(e)	at	(2,.75)	{$e$};
	\node	(f)	at	(0,0)	{$f$};
	\node	(g)	at	(1,0)	{$g$};
	%\node	(h)	at	(2,0)	{$h$};
	\node	(i)	at	(3,.75)	{$i$};
	\node	(j)	at	(4,1.5)	{$j$};
	\node	(k)	at	(4,.75)	{$k$};
	\node	(l) at	(4,0)	{$l$};
	\draw (1,1.5) ellipse [x radius=1.4, y radius=0.3];
	\draw (0.5,0) ellipse [x radius=0.9, y radius=0.3];
	\draw (3,0.75) ellipse [x radius=1.4, y radius=0.3];
	\draw[rotate around={90:(0,.75)}] (0,.75) ellipse [x radius=1.1, y radius=0.3];
	\draw[rotate around={90:(2,1.125)}] (2,1.125) ellipse [x radius=0.7, y radius=0.3];
	\draw[rotate around={90:(4,.75)}] (4,.75) ellipse [x radius=1.1, y radius=0.3];
	\node	(w1)	at	(1,2)	{\small $w_1$};
	\node	(w2)	at	(1,-.5)	{\small $w_2$};
	\node	(w3)	at	(3,.25)	{\small $w_3$};
	\node	(w4)	at	(.1,-.6)	{\small $w_4$};
	\node	(w5)	at	(2.1,.15)	{\small $w_5'$};
	\node	(w6)	at	(4.1,-.6)	{\small $w_6$};
\end{tikzpicture}
		\label{fig:hg-acyclic}
	}\stdhspace%
	\subfigure[Hypergraph $H_{P'}$ of $P'$.]{
		\begin{tikzpicture}[scale=1]
	\node	(a)	at	(0,1.5)	{$a$};
	\node	(b)	at	(1,1.5)	{$b$};
	\node	(c)	at	(2,1.5)	{$c$};
	\node	(d)	at	(0,.75)	{$d$};
	\node	(e)	at	(2,.75)	{$e$};
	\node	(f)	at	(0,0)	{$f$};
	\node	(g)	at	(1,0)	{$g$};
	\node	(h)	at	(2,0)	{$h$};
	\node	(i)	at	(3,.75)	{$i$};
	\node	(j)	at	(4,1.5)	{$j$};
	\node	(k)	at	(4,.75)	{$k$};
	\node	(l) at	(4,0)	{$l$};
	\draw (1,1.5) ellipse [x radius=1.4, y radius=0.3];
	\draw (1,0) ellipse [x radius=1.4, y radius=0.3];
	\draw (3,0.75) ellipse [x radius=1.4, y radius=0.3];
	\draw[rotate around={90:(0,.75)}] (0,.75) ellipse [x radius=1.1, y radius=0.3];
	\draw[rotate around={90:(2,.75)}] (2,.75) ellipse [x radius=1.1, y radius=0.3];
	\draw[rotate around={90:(4,.75)}] (4,.75) ellipse [x radius=1.1, y radius=0.3];
	\node	(w1)	at	(1,2)	{\small $w_1$};
	\node	(w2)	at	(1,-.5)	{\small $w_2$};
	\node	(w3)	at	(3,.25)	{\small $w_3$};
	\node	(w4)	at	(.1,-.6)	{\small $w_4$};
	\node	(w5)	at	(2.1,-.6)	{\small $w_5$};
	\node	(w6)	at	(4.1,-.6)	{\small $w_6$};
\end{tikzpicture}
		\label{fig:hg-cy}
	}
	\caption{The hypergraphs corresponding to the two puzzles $P,P'$ of Figure~\ref{fig:crossword}.}
	\label{fig:hypergraphs}
\end{figure}

\subsection{Generalized Hypertree Decompositions}
We use $B(E)$ to denote the set of vertices of $H$ \emph{covered} by a certain set of edges $E$ of $H$.
More precisely, given a hypergraph $H = (V(H),E(H))$ and a set of edges $E \subseteq E(H)$, we define $B(E) = \bigcup_{e \in E} e$  as the set of all vertices of $H$ contained in the set of edges $E$.

A \emph{generalized hypertree decomposition} (GHD)~\cite{DBLP:journals/jcss/GottlobLS02} of a hypergraph  $H=(V(H),E(H))$ is a tuple 
$\left< T, (B_u)_{u\in T}, (\lambda_u)_{u\in T} \right>$ where $T = (N(T),E(T))$ is a tree, every $B_u$
is a subset of $V(H)$, every $\lambda_u$ is a subset of $E(H)$, and the following hold:
\begin{enumerate}
	\item[(1)] For every edge $e \in E(H)$, there is a node $u$ in $T$, such that $e \subseteq B_u$, and
	\item[(2)] for every vertex $v \in V(H)$,  $\{u \in T \mid v \in B_u\}$ is a connected subtree in $T$, and
	\item[(3)] for each $u\in T$,  $B_u \subseteq  B(\lambda_u)$ holds.
\end{enumerate}

We refer to the vertex sets $B_u$ as the \emph{bags} of the GHD, while we call the edge sets
$\lambda_u$ \emph{edge covers}. By slight abuse of notation, we write $u \in T$ to express that $u$ is a node in $N(T)$.
Condition (2) is also called the \emph{connectedness condition}.

We use the following notational conventions. To avoid confusion, we will
consequently refer to the elements in $V(H)$ as {\em vertices\/} of the hypergraph and to the
elements in $N(T)$ as the {\em nodes\/} of the decomposition. For a node $u \in T$, we write $T_u$
to denote the subtree of $T$ rooted at $u$. By slight abuse of notation, we will often write $u' \in
T_u$ to denote that $u'$ is a node in the subtree $T_u$ of $T$.
Finally, we define $V(T_u) = \bigcup_{u' \in T_u} B_{u'}$.

The \emph{width} of a GHD is defined as the largest size of any set $\lambda_u$ over all nodes $u \in T$.
The generalized hypertree width of $H$ ($\ghw(H)$) is the minimum width over all GHDs of $H$.

\begin{example}
	A GHD for the hypergraph $H_{P'}$ of Figure~\ref{fig:hg-cy} is shown in Figure~\ref{fig:puzzle-cy-ghd}.
	It is easy to check that conditions (1)-(3) are satisfied.
	This GHD has width 2 because there is at least one node of the decomposition with $\lvert \lambda_u \rvert = 2$.
	It is also possible to prove that $\ghw(H_{P'}) = 2$. Indeed, the hypergraph $H_{P'}$ has a cycle and, thus, it cannot have width 1.
	On the contrary, the hypergraph $H_P$ is acyclic and $\ghw(H_{P}) = 1$ holds.
\end{example}

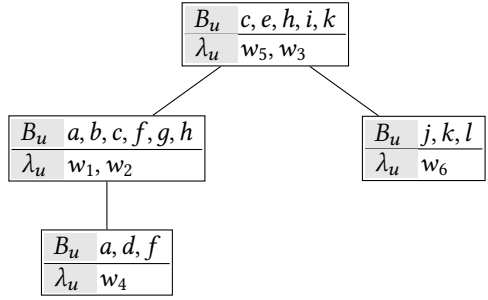
\begin{figure}
	\centering
	\tikzstyle{htmatrix}=[%
		matrix,
		matrix of nodes,
		nodes={draw=none},
		column 1/.style={nodes={fill=gray!20,align=right},anchor=base east,minimum width=2em},
		column 2/.style={anchor=base west},
		ampersand replacement=\&
	]
	
	\begin{tikzpicture}[sibling distance=12em,
		every node/.style = {shape=rectangle,draw,inner sep=1,align=center}
		]
		\node (root) [htmatrix] {\htnode{$c,e,h,i,k$}{$w_5,w_3$}}
		child {
			node (c1) [htmatrix] {\htnode{$a,b,c,f,g,h$}{$w_1,w_2$}} 
			child {
				node (c2) [htmatrix]{\htnode{$a,d,f$}{$w_4$}}
			}
		}
		child {
			node (c3) [htmatrix] {\htnode{$j,k,l$}{$w_6$}}
		};
		%\node (circ1) [shape=circle, style={color=red}, thick] at ([yshift=-6,xshift=-16]root) {\;\;\;};
		%\node (circ2) [shape=circle, style={color=red}, thick] at ([yshift=5,xshift=-10]c2) {\;\;\;};
	\end{tikzpicture}
		
	\caption{A GHD of width 2 for the hypergraph $H_{P'}$ of Figure~\ref{fig:hg-cy}.}
	\label{fig:puzzle-cy-ghd}
\end{figure}

\subsection{Top-Down Construction of GHDs}
\label{sect:ComputingHDs}

For a set $U \subseteq V(H)$ of vertices, we define $[U]$-components of a hypergraph $H$ through the following steps.
\begin{itemize}
	\item We define {\em $[U]$-adjacency\/} as a binary relation on $E(H)$ as follows: 
	two edges $e_1$ and $e_2$ are {\em $[U]$-adjacent\/}, if $(e_1 \cap e_2) \setminus U \neq \emptyset$ holds.
	\item We define {\em $[U]$-connectedness\/} as the transitive closure of the {\em $[U]$-adjacency\/} relation.
	\item A {\em $[U]$-component\/} of $H$ is a maximally $[U]$-connected subset $C \subseteq E(H)$.
\end{itemize}
For a set of edges $S \subseteq E(H)$, we say that $C$ is \emph{$[S]$-connected} as a short-cut for $C$ is \emph{$[U]$-connected} with $U = \bigcup_{e \in S} e$.
We also call $S$ a \emph{separator} and $C$ an \emph{$[S]$-component}.
The \emph{size of an $[S]$-component} $C$ is defined as the number of edges in $C$.

\SetAlCapFnt{\large}
\SetAlCapNameFnt{\large}

\begin{algorithm}[tb]
\DontPrintSemicolon
\SetKwInOut{KwPara}{Parameter}
\SetKwInput{KwType}{Type}

\KwType{Node=(Bag: Vertex Set, Edge Cover: Set of Edges, Children: Set of Nodes)}
\KwIn{$H$: Hypergraph}
\KwPara{\ $k$: width parameter}
\KwOut{a GHD of $H$ with width $\leq k$, or \textbf{Reject} if none exists}

 \SetKwFunction{algo}{FindDecomp} 
 \SetKwFunction{getComp}{ComputeComponents}
 \SetKwFunction{checkScenery}{SceneCreation}
 \SetKwFunction{windUp}{SceneCreationUp}
 \SetKwFunction{windDown}{SceneCreationDown}
 \SetKwFunction{buildGHD}{BuildGHD}
 \SetKwFunction{decomp}{$D$}
 \SetKwProg{myalg}{function}{}{}
 \Begin{\textbf{return} \algo{$H$} \Comment{initial call} \\ }
\myalg{\algo{$H'$: Subgraph} }{

    \For(\Comment{iterate over set of possible bags relative to $H'$ with cover number $\leq k$}){ \texttt{sep }$\in \text{Separators}(H',k)$   } { 

    \label{line:sepIter}
    \texttt{cover} $\coloneqq $ find some \texttt{cover} $\subseteq E(H)$ 
    such that \texttt{sep} $ \subseteq B(\text{cover})$ and $\abs{\text{\texttt{cover}}} \leq k$ \;
    
    \texttt{comps} $\coloneqq $ $[sep]$-components of $H'$  \;
    \label{line:compGen}
    \texttt{children} $\coloneqq \emptyset$ \Comment{initialize the set of children} \\

    \For{ $c\in$ \texttt{comps}} { 
        \label{line:schemaLoopStart}
        \If {  \algo{$c$}  returns \textbf{Reject}} {
            \textbf{continue} outer loop
        }

        \texttt{children} $=$ \texttt{children}  $\cup $ \algo{$c$} \Comment{recursive call to compute subtrees} \\
      }     
      \label{line:schemaLoopEnd}

    \textbf{return} Node(Bag : \texttt{sep}, Edge Cover: \texttt{cover}, Children: \texttt{children})
    \label{line:ReturnGHD}

    }
    \textbf{return} \textbf{Reject} \Comment{reject if no valid separator could be found} \\
}

\caption{A schematic top-down GHD algorithm}
\label{alg:topdown}
\end{algorithm}

Many top-down GHD algorithms such as~\cite{10.1145/3440015,DBLP:conf/ijcai/GottlobOP20,DBLP:journals/jea/GottlobS08} can be described by the same schema.
Here we simply refer to a generic top-down schema with the understanding that each implementation presents its own peculiarities. A pseudo-code description of such a schematic top-down GHD algorithm is given in Algorithm~\ref{alg:topdown}.

For a fixed $k \geq 1$, Algorithm~\ref{alg:topdown} takes as input a hypergraph $H$ and builds a GHD of $H$ of width $\leq k$ or rejects if none can be found.
Each call of the function \texttt{FindDecomp} is over a particular set of edges, referred to as $H'$. This recursive function forms the core of this schematic description. A recursive function is chosen for notational simplicity, a real implementation could also just use nested loops or other ways to implement iteration. 
The function \texttt{FindDecomp} iterates over a set of valid separators, as seen in line~\ref{line:sepIter}. Note that the particulars of how this set is computed can vary greatly between specific implementations. The algorithm picks a vertex set \texttt{sep} from this set  and then finds a subset of edges \texttt{cover} such that $\abs{\texttt{cover}} \leq k$ and $\text{\texttt{sep}} \subseteq B(\text{\texttt{cover}})$. These two will form the bag (resp. edge cover) of the root node $u$ of the GHD found for subgraph $H'$.

Next, in line 6, it proceeds to determine all $[\text{\texttt{sep}}]$-components of $H'$, which we denote as $C_1, \dots, C_\ell$.
It was shown in~\cite{DBLP:journals/jcss/GottlobLS02} that if $H$ has a GHD of width $\leq k$, then it also has a
GHD of width  $\leq k$ such that the edges in each $C_i$ are ``covered'' in separate and distinct
subtrees below $u$.
We say that such a GHD is in \emph{normal form}.
More precisely, ``covered'' means that $u$ has $\ell$ child nodes $u_1,
\dots, u_\ell$, s.t. for every $i$ and every $e \in C_i$, there exists a node $u_e \in T_{u_i}$
with $e \subseteq B_{u_e}$. Hence, the algorithm recursively searches for a GHD of the
subhypergraphs $H_i$ with $E(H_i) = C_i$ and   $V(H_i) = \bigcup C_i$. We can see this inside function \texttt{FindDecomp}  in lines \ref{line:schemaLoopStart} to \ref{line:schemaLoopEnd}. If all recursive calls succeed, the function terminates by forming a GHD with root $u$ and subtrees covering the components $C_i$, seen in line~\ref{line:ReturnGHD}.

\begin{example}\label{ex:ghd-decomp}
	We decompose the hypergraph $H_{P'}$ of Figure~\ref{fig:hg-cy} so to obtain the GHD shown in Figure~\ref{fig:puzzle-cy-ghd}.
	For $k=2$, the generic top-down algorithm $\mathcal{A}$ takes in input $H_{P'}$ and computes a GHD of $H_{P'}$ of width 2, if it exists.
	Firstly, $\mathcal{A}$ guesses a separator of size $\leq 2$ with edges in $C = E(H_{P'})$, which will be used as the edge cover $\lambda_u$ for the root node $u$ of the GHD.
	Suppose that $\lambda_u = \{w_5,w_3\}$ and $B_u = \{c,e,h,i,k\}$ is a suitable choice for the bag.
	Then the root $u$ of the decomposition will be exactly the same of Figure~\ref{fig:puzzle-cy-ghd}.
	At this point, $\mathcal{A}$ computes the $[\lambda_u]$-components $C_1,C_2 \subseteq C$ with $C_1 = \{w_1,w_2,w_4\}$, $C_2 = \{w_6\}$ (Figure~\ref{fig:components}).
	Each $C_i$ is now recursively decomposed.
	Since $\lvert C_2 \rvert \leq 2$, it can be ``covered'' by a single node of the decomposition.
	On the contrary, $\lvert C_1 \rvert > 2$, thus $\mathcal{A}$ starts a recursive call on $C_1$ and guesses the separator $S = \{w_1,w_2\}$ and computes the new bag $\{a,b,c,f,g,h\}$.
	As the component $C_1$ is split w.r.t. $S$, we obtain a single component $C_3 \subseteq C_1$, with $C_3 = \{w_4\}$.
	This component can be ``covered'' by a single node.
	Finally, all nodes are attached to their respective fathers except for the root.
	The resulting decomposition is returned.
\end{example}

\begin{figure}[t]
	%\captionsetup[subfigure]{justification=centering}
	\centering
	%\begin{subfigure}[b]{0.4\textwidth}
	\subfigure[Component $C_1$.]{
		%\centering
		\begin{tikzpicture}[scale=1]
			\node	(a)	at	(0,1.5)	{$a$};
			\node	(b)	at	(1,1.5)	{$b$};
			\node	(c)	at	(2,1.5)	{$c$};
			\node	(d)	at	(0,.75)	{$d$};
			\node[text opacity=.5]	(e)	at	(2,.75)	{$e$};
			\node	(f)	at	(0,0)	{$f$};
			\node	(g)	at	(1,0)	{$g$};
			\node	(h)	at	(2,0)	{$h$};
			\node[text opacity=.5]	(i)	at	(3,.75)	{$i$};
			\node[text opacity=.5]	(j)	at	(4,1.5)	{$j$};
			\node[text opacity=.5]	(k)	at	(4,.75)	{$k$};
			\node[text opacity=.5]	(l) at	(4,0)	{$l$};
			\draw (1,1.5) ellipse [x radius=1.4, y radius=0.3];
			\draw (1,0) ellipse [x radius=1.4, y radius=0.3];
			\draw[opacity=.5,dashed] (3,0.75) ellipse [x radius=1.4, y radius=0.3];
			\draw[rotate around={90:(0,.75)}] (0,.75) ellipse [x radius=1.1, y radius=0.3];
			\draw[rotate around={90:(2,.75)},opacity=.5,dashed] (2,.75) ellipse [x radius=1.1, y radius=0.3];
			\draw[rotate around={90:(4,.75)},opacity=.5,dashed] (4,.75) ellipse [x radius=1.1, y radius=0.3];
			\node	(w1)	at	(1,2)	{\small $w_1$};
			\node	(w2)	at	(1,-.5)	{\small $w_2$};
			%\node	(w3)	at	(3,.25)	{\small $w_3$};
			\node	(w4)	at	(.1,-.6)	{\small $w_4$};
			%\node	(w5)	at	(2.1,-.6)	{\small $w_5$};
			%\node	(w6)	at	(4.1,-.6)	{\small $w_6$};
		\end{tikzpicture}
		%\caption{Component $C_1$.}
		%\label{fig:hg-acyclic}
	}\stdhspace%
	%\end{subfigure}\hspace{2em}
	%\begin{subfigure}[b]{0.4\textwidth}
	\subfigure[Component $C_2$.]{
		%\centering
		\begin{tikzpicture}[scale=1]
			\node[text opacity=.5]	(a)	at	(0,1.5)	{$a$};
			\node[text opacity=.5]	(b)	at	(1,1.5)	{$b$};
			\node[text opacity=.5]	(c)	at	(2,1.5)	{$c$};
			\node[text opacity=.5]	(d)	at	(0,.75)	{$d$};
			\node[text opacity=.5]	(e)	at	(2,.75)	{$e$};
			\node[text opacity=.5]	(f)	at	(0,0)	{$f$};
			\node[text opacity=.5]	(g)	at	(1,0)	{$g$};
			\node[text opacity=.5]	(h)	at	(2,0)	{$h$};
			\node[text opacity=.5]	(i)	at	(3,.75)	{$i$};
			\node	(j)	at	(4,1.5)	{$j$};
			\node	(k)	at	(4,.75)	{$k$};
			\node	(l) at	(4,0)	{$l$};
			\draw[opacity=.5,dashed] (1,1.5) ellipse [x radius=1.4, y radius=0.3];
			\draw[opacity=.5,dashed] (1,0) ellipse [x radius=1.4, y radius=0.3];
			\draw[opacity=.5,dashed] (3,0.75) ellipse [x radius=1.4, y radius=0.3];
			\draw[rotate around={90:(0,.75)},opacity=.5,dashed] (0,.75) ellipse [x radius=1.1, y radius=0.3];
			\draw[rotate around={90:(2,.75)},opacity=.5,dashed] (2,.75) ellipse [x radius=1.1, y radius=0.3];
			\draw[rotate around={90:(4,.75)}] (4,.75) ellipse [x radius=1.1, y radius=0.3];
			\phantom{\node	(w1)	at	(1,2)	{\small $w_1$};}
			%\node	(w2)	at	(1,-.5)	{\small $w_2$};
			%\node	(w3)	at	(3,.25)	{\small $w_3$};
			%\node	(w4)	at	(.1,-.6)	{\small $w_4$};
			%\node	(w5)	at	(2.1,-.6)	{\small $w_5$};
			\node	(w6)	at	(4.1,-.6)	{\small $w_6$};
		\end{tikzpicture}
		%\caption{Component $C_2$.}
		%\label{fig:hg-cy}
	}
	%\end{subfigure}
	\caption{The two components $C_1, C_2$ of Example~\ref{ex:ghd-decomp} obtained by removing $\{w_5,w_3\}$ from $H_{P'}$.}
	\label{fig:components}
\end{figure}
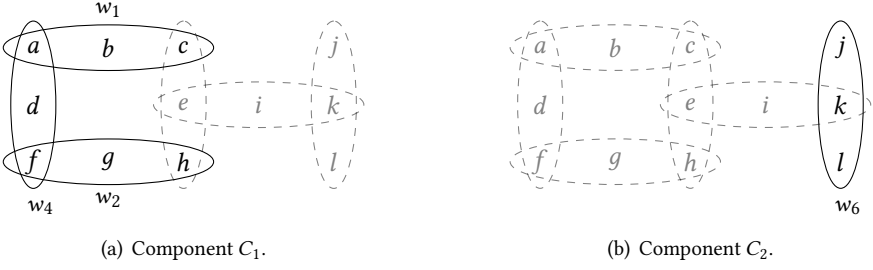

\section{The GHD Update Problem}
\label{sec:update-problem}
We introduce the problem of updating GHDs upon constraint modifications.
Here we propose important classes of elementary modifications and describe their effect on both the hypergraph and the CSP.
Then, we formally define the update problem and settle its complexity.

\subsection{Elementary Constraint Modifications}
\label{sec:mods}
The modification of a CSP affects its underlying hypergraph and consequently its GHDs.
We define a framework that allows us to update the GHD of a CSP in the face of constraint modifications.
First of all, we define what a modification is.

\begin{definition}
	\label{def:mod}
	Given a hypergraph $H$, a modification is a function mapping hypergraphs to hypergraphs.
\end{definition}

We identify three fundamental hypergraph objects for the computation of a GHD: vertices, edges, and intersections between edges.
As modifying a CSP typically implies the modification of these objects, we define six classes of modifications that reflect elementary changes on a hypergraph.
In this sense, the proposed classes of modifications are natural, even though not necessarily minimal.
We show the effect of these modifications on the hypergraph and intuitively explain their correspondence with the related CSP.
On the other hand, we ignore all those CSP modifications that do not change the hypergraph.
For instance, changing constraint relations does not affect the CSP structure, therefore the same GHD can be reused to solve the modified CSP again.
We start with two classes for vertex modifications.
\begin{definition}
	\UnConst{} is the class of modifications $\delta$ such that for every hypergraph $H$:
	\begin{itemize}
		\item $V(\delta(H)) \coloneqq V(H) \cup w$,
		\item $E(\delta(H)) \coloneqq E \cup \{e' \cup \{w\} \mid e' \in E' \}$, where $(E,E')$ is a partition of $E(H)$ and $E' \neq \emptyset$.
	\end{itemize}
	\Const{} is the class of modifications $\delta$ such that for every hypergraph $H$,
	$V(\delta(H)) \coloneqq V(H) \setminus \{v\}$ and $E(\delta(H)) \coloneqq \{ e \setminus \{v\} \mid e \in E(H) \}$.
\end{definition}
On the CSP level, \Const{} functions represent the binding of a CSP variable to a constant value and could possibly simplify its hypergraph.
On the other hand, \UnConst{} functions remove such binding, i.e., replace a constant with a variable.

Two classes are needed for edge insertion and removal.
\begin{definition}
	\AddRel{} is the class of modifications $\delta$ s.t. for every hypergraph $H$,
	$\delta(H) \coloneqq H \cup f$, where $f \notin E(H)$ is a new edge.
	\DelRel{} is the class of modifications $\delta$ s.t. for every hypergraph $H$,
	$\delta(H) \coloneqq H \setminus e$, where $e \in E(H)$. 
\end{definition}
\AddRel{} and \DelRel{} modifications correspond to alterations of the set of constraints $C_t$ of a CSP.
In particular, $\delta \in \AddRel{}$ introduces a new constraint in $C_t$, while $\delta \in \DelRel{}$ removes a constraint from $C_t$.

Finally, we present classes to modify intersections between edges.
Let $H$ be a hypergraph. Given $U \subseteq V(H)$, we denote with $E_U = \{ e \in E(H) \mid e \cap U \neq \emptyset\}$ the edges incident on $U$.
\begin{definition}
	\JJoin{} is the class of modifications $\delta$ such that for every hypergraph $H$,
	some vertices $U \subseteq V(H)$ are merged into $w \in V(\delta(H))$ and the edges in $E_U$ are incident on $w$.
	\UnJoin{} is the class of modifications $\delta$ such that for every hypergraph $H$,
	a vertex $w \in V(H)$ is split into a set $U \subseteq V(\delta(H))$ and the edges in $E_{\{w\}}$ are arbitrarily distributed on $U$.
\end{definition}
Intuitively, an \JJoin{} modification introduces an equality constraint between some variables of the CSP.
In other words, a new \emph{AllEqual} constraint is defined over a set of variables of the CSP.
On the other hand, \UnJoin{} modifications remove this kind of constraint and thus all equalities between a specific set of variables.

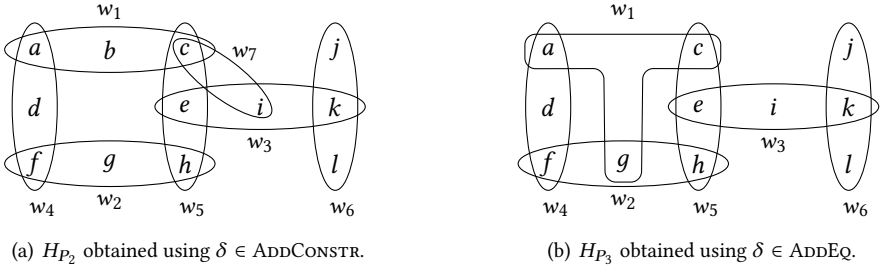
\begin{figure}[t]
	\centering
	\subfigure[$H_{P_2}$ obtained using $\delta \in \AddRel{}$.]{
		\begin{tikzpicture}[scale=1]
	\node	(a)	at	(0,1.5)	{$a$};
	\node	(b)	at	(1,1.5)	{$b$};
	\node	(c)	at	(2,1.5)	{$c$};
	\node	(d)	at	(0,.75)	{$d$};
	\node	(e)	at	(2,.75)	{$e$};
	\node	(f)	at	(0,0)	{$f$};
	\node	(g)	at	(1,0)	{$g$};
	\node	(h)	at	(2,0)	{$h$};
	\node	(i)	at	(3,.75)	{$i$};
	\node	(j)	at	(4,1.5)	{$j$};
	\node	(k)	at	(4,.75)	{$k$};
	\node	(l) at	(4,0)	{$l$};
	\draw (1,1.5) ellipse [x radius=1.4, y radius=0.3];
	\draw (1,0) ellipse [x radius=1.4, y radius=0.3];
	\draw (3,0.75) ellipse [x radius=1.4, y radius=0.3];
	\draw[rotate around={90:(0,.75)}] (0,.75) ellipse [x radius=1.1, y radius=0.3];
	\draw[rotate around={90:(2,.75)}] (2,.75) ellipse [x radius=1.1, y radius=0.3];
	\draw[rotate around={90:(4,.75)}] (4,.75) ellipse [x radius=1.1, y radius=0.3];
	\draw[rotate around={-36.9:(2.5,1.125)}] (2.5,1.125) ellipse [x radius=.8, y radius=0.25];
	\node	(w1)	at	(1,2)	{\small $w_1$};
	\node	(w2)	at	(1,-.5)	{\small $w_2$};
	\node	(w3)	at	(3,.25)	{\small $w_3$};
	\node	(w4)	at	(.1,-.6)	{\small $w_4$};
	\node	(w5)	at	(2.1,-.6)	{\small $w_5$};
	\node	(w6)	at	(4.1,-.6)	{\small $w_6$};
	\node	(w7)	at	(2.8,1.45)	{\small $w_7$};
\end{tikzpicture}
		\label{fig:mod-add-edge}
	}\stdhspace%
	\subfigure[$H_{P_3}$ obtained using $\delta \in \JJoin{}$.]{
		\begin{tikzpicture}[scale=1]
	\node	(a)	at	(0,1.5)	{$a$};
	\node	(b)	at	(1,1.5)	{\phantom{$b$}};
	\node	(c)	at	(2,1.5)	{$c$};
	\node	(d)	at	(0,.75)	{$d$};
	\node	(e)	at	(2,.75)	{$e$};
	\node	(f)	at	(0,0)	{$f$};
	\node	(g)	at	(1,0)	{$g$};
	\node	(h)	at	(2,0)	{$h$};
	\node	(i)	at	(3,.75)	{$i$};
	\node	(j)	at	(4,1.5)	{$j$};
	\node	(k)	at	(4,.75)	{$k$};
	\node	(l) at	(4,0)	{$l$};
	%
	%\draw (1,1.5) ellipse [x radius=1.4, y radius=0.3];
	\draw (1,0) ellipse [x radius=1.4, y radius=0.3];
	\draw (3,0.75) ellipse [x radius=1.4, y radius=0.3];
	\draw[rotate around={90:(0,.75)}] (0,.75) ellipse [x radius=1.1, y radius=0.3];
	\draw[rotate around={90:(2,.75)}] (2,.75) ellipse [x radius=1.1, y radius=0.3];
	\draw[rotate around={90:(4,.75)}] (4,.75) ellipse [x radius=1.1, y radius=0.3];
	\draw[rounded corners] ($(a.north west) + (-.1,0)$) -- ($(c.north east) + (.1,0)$) -- ($(c.south east) + (.1,-.05)$)
		-- ($(b.south east) + (.05,0)$) -- ($(g.south east) + (.035,0)$) -- ($(g.south west) + (-.035,0)$)
		-- ($(b.south west) + (-.05,0)$) -- ($(a.south west) + (-.1,-.05)$)
		-- cycle;%e7
	\node	(w1)	at	(1,2)	{\small $w_1$};
	\node	(w2)	at	(1,-.5)	{\small $w_2$};
	\node	(w3)	at	(3,.25)	{\small $w_3$};
	\node	(w4)	at	(.1,-.6)	{\small $w_4$};
	\node	(w5)	at	(2.1,-.6)	{\small $w_5$};
	\node	(w6)	at	(4.1,-.6)	{\small $w_6$};
\end{tikzpicture}
		\label{fig:mod-add-join}
	}
	\caption{Hypergraphs obtained by applying to $H_{P'}$ the modifications described in Example~\ref{ex:more-mods}.}
	\label{fig:more-mods}
\end{figure}

\begin{example}\label{ex:more-mods}
	The hypergraph $H_{P'}$ of Figure~\ref{fig:hg-cy} is obtained by applying a modification $\delta \in \UnConst{}$ to $H_P$ in Figure~\ref{fig:hg-acyclic}.
	In particular, $\delta$ adds a new vertex $h$ in the edges $\{f,g\}$ and $\{c,e\}$ as well as in $V(H_P)$.
	Note that $H_P$ can be obtained from $H_{P'}$ via a modification $\delta \in \Const{}$ removing $h$ from $V(H_{P'})$.
	
	Figure~\ref{fig:more-mods} shows two additional modifications of $H_{P'}$.
	The hypergraph $H_{P_2}$ of Figure~\ref{fig:mod-add-edge} is the result of a modification $\delta \in \AddRel{}$ introducing a new edge $\{c,i\}$,
	while $H_{P_3}$ of Figure~\ref{fig:mod-add-join} shows the effect on $H_{P'}$ of a $\delta \in \JJoin{}$ adding an \emph{AllEqual} constraint between the variables $b,g$ ($b$ is merged into $g$).
	Finally, $H_{P'}$ can be obtained through an appropriate inverse modification $\delta$ to $H_{P_2}$ and $H_{P_3}$ with $\delta \in \DelRel{}$ and $\delta \in \UnJoin{}$, respectively.
\end{example}

Note that the set of all considered elementary modifications is \emph{complete},
i.e., given two hypergraphs $H,H'$, there exists a sequence of modifications $\delta_1,\dots,\delta_\ell$ s.t. $H' = \delta_\ell(\cdots(\delta_1(H)))$.

\subsection{The Complexity of Updating GHDs}
\label{sec:complex}
Recall that checking $\ghw(H) \leq k$, and therefore computing a width $k$ GHD of $H$, is \np-hard even when $k>1$ is constant~\cite{DBLP:conf/pods/FischlGP18,DBLP:journals/jacm/GottlobMS09,DBLP:journals/jacm/GottlobLPR21}.
In the context of modifications this naturally presents the question of the complexity of the following task: given a hypergraph $H$ together with a minimal width GHD, as well as a modification $\delta$, find a GHD for $\delta(H)$ with the same width if one exists, or correctly identify that the width increased. Intuitively, the knowledge of a witness for $H$ could make the problem easier, in particular if $\delta$ is a simple modification. 
Formally, we extend the standard problem of checking $\ghw(H) \leq k$ for constant $k$ (see e.g.,~\cite{DBLP:conf/pods/FischlGP18}) by simply adding a modification (from some class of modifications $\Delta$) and a GHD of the original hypergraph to the input.

\begin{center}
	\begin{problemdef}[framed]{\minupdateProb{}($\Delta$)}
		Instance: & hypergraph $H$, modification $\delta \in \Delta$, a minimal width GHD of $H$ \\
		Output: & A GHD of $\delta(H)$ with width $\leq \ghw(H)$ if it exists\\ & or answer 'no' otherwise.
	\end{problemdef}
\end{center}

Importantly, \minupdateProb{} is a search problem rather than a decision problem. This is motivated from two sides. Our primary motivation stems from practical situation in which small, iterative updates are consistently made to some CSP and for which we want to maintain a low width GHD. Since the GHD is necessary to possibly exploit low width for solving the CSP, we are interested in the search problem rather than the decision problem.

The second motivating factor comes from the possibility of certain classes of modifications capturing other classes, i.e., if one can express some modification in a class $\Delta$ via a sequence of  modifications from another class $\Delta'$. Focusing purely on the decision problem makes it problematic to consider the complexity of sequences of updates, since we have no information on the complexity of obtaining the new input GHDs along the sequence of updates.  By studying the search problem instead we can make strong statements for such cases.

The complexity of search problems is a complex topic and the full theoretical framework is not necessary in our context here. Instead, we will be content with showing that even for the simple classes of atomic updates that were discussed previously (except \Const{}), \minupdateProb can not be solved in polynomial time. 
Note that \minupdateProb{} trivially reduces to the problem of finding an optimal GHD of $\delta(H)$ and all negative results therefore extend also to finding optimal GHDs under modifications.

\begin{theorem}
	\label{thm:hard}
	For $\Delta \in \{$\JJoin{}, \UnJoin{}, \UnConst{}, \AddRel{}, \DelRel{}$\}$,  \minupdateProb{}($\Delta$) cannot be solved in polynomial time (assuming $\ptime \neq \np$).
\end{theorem}
\begin{proof}[Proof Idea]
	The basic strategy for each modification class $\Delta$ is simple. We show how to decide an \np-hard decision problem by finding an initial hypergraph $H_0$, which can be modified by some sequence of $\delta_1,\delta_2,\dots,\delta_\ell \in \Delta$ to some target $H$. The decision problem will be equivalent to the question whether $\ghw(H) \leq \ghw(H_0)$. However, this strategy presents us with two technical challenges. First, the initial $H_0$ needs to be chosen in such a way that a minimal width GHD can be constructed in polynomial time. Second, there can be no index $i$ such that after applying the first $i$ modifications to $H_0$, we get a $H_i$ with $\ghw(H_0) < \ghw(H_i)$ even when $\ghw(H) \leq \ghw(H_0)$. That is, the sequence can not increase the width at intermediate hypergraphs before decreasing again.
	
	To tackle these issues we do not reduce from \ghw checking where the second challenge is particularly problematic as our operations are not monotonic (wrt. \ghw) in general. Instead, we reduce from \textsc{3-Sat} by building on the proof of \np-hardness of \ghw checking (for constant width) given by~\citet{DBLP:journals/jacm/GottlobLPR21}. There a hypergraph is constructed that has \ghw 3 exactly if some \textsc{3-Sat} instance is satisfiable, and has \ghw 2 otherwise. Using this specific hypergraph we give concrete $H_0$ and modification sequences as described above for each $\Delta$. A full proof for each $\Delta$ is given in Appendix~\ref{sec:newproof}.
\end{proof}

Updating GHDs is computationally difficult for all of the
natural atomic operations that we considered, except for \Const{}\footnote{A \Const{} modification results in an induced subhypergraph, which is well known to never increase in width (see e.g.,~\cite{DBLP:journals/jacm/GottlobLPR21}).} (where the problem is trivial as \Const{} cannot increase width and a new GHD is trivial to construct). As
part of the proof of Theorem~\ref{thm:hard} we discuss how to decide
\textsc{3-Sat} via sequences of modifications, as long as those
sequences adhere to certain conditions. Using this observation we can strengthen the statement from Theorem~\ref{thm:hard} to all modification classes that capture any of the hard atomic cases in the following formal sense.

For a sequence of modifications $\delta_1, \delta_2,\dots, \delta_\ell$ let us write $\delta_1^n(H)$ as a shorthand for $\delta_n(\delta_{n-1}(\cdots(\delta_1(H))\cdots))$.
Let $\Delta$, $\Delta'$ be two sets of modifications. We say that $\Delta$ \emph{polynomially captures} $\Delta'$ if for every hypergraph $H$ and $\delta' \in \Delta'$ there exists a sequence $\delta_1,\delta_2,\dots, \delta_\ell$ of modifications in $\Delta$ such that $\ghw(\delta_1^i(H)) \leq \ghw(\delta_1^{i+1}(H))$ for $1 \leq i < \ell$, $\delta_1^\ell(H) = \delta'(H)$ and $\ell$ is polynomially bounded in the size of $H$. In plain terms, every modification in $\Delta'$ can equivalently be reached via a polynomial sequence of modifications from $\Delta$.

\begin{corollary}
	Let $\Delta$ be a class of modifications that polynomially captures at least one of \JJoin{}, \UnJoin{}, \UnConst{}, \AddRel{}, or \DelRel{}.
	Then  \minupdateProb{}($\Delta$) cannot be solved in polynomial time (assuming $\ptime \neq \np$).
\end{corollary}

\section{Towards a Framework for Updates}
\label{sec:update-practice}

We have seen that \minupdateProb{} is difficult in general.
In the following we thus focus on making the first steps towards practical solutions for the problem. In this section we present the theoretical framework of \emph{\magic subtrees} for the uniform treatment of GHD updates under arbitrary modifications.
Moreover, we briefly discuss how our approach extends to sequences of elementary modifications.

\subsection{The $\delta$-\magic Subtrees of a Decomposition}

We lay the theoretical foundations of $\delta$-\magic subtrees,
a notion that will let us treat updates uniformly.
We first introduce some convenient notation.
Let \defGHD{} be a GHD of a CSP $P$ and let $T'$ be a subtree of $T$. We write $T \setminus T'$ for the forest created by removing the nodes of $T'$ from $T$.
Since we are
interested in the hypergraph structure,
we write $H[T']$ for the  subhypergraph of $H$ induced by the vertices $\bigcup_{u \in T'}B_u$.

We are now ready to introduce the central notion of our framework, $\delta$-\magic subtrees.
Intuitively, these subtrees (of a decomposition) represent a kind of \emph{local neighborhood} of the modification $\delta$, i.e., the segment of the decomposition that corresponds to those parts of the hypergraph that are changed by $\delta$. Note that the definitions and results in this section apply not only to the previously discussed elementary modifications but to arbitrary modifications in the sense of Definition~\ref{def:mod}.

\begin{definition}[$\delta$-\magic subtree]
	\label{def:magic}
	Let $\mathcal{G}$ be a GHD of hypergraph $H$ with tree $T$, and let $\delta$ be a
	modification.
	A subtree $T^*$ of $T$ is a \emph{$\delta$-\magic{} subtree}
	if the following conditions hold:
	\begin{itemize}
		\item \(H[T\setminus T^*] = \delta(H)[T\setminus T^*]\), 
		\item and no $v \in V(\delta(H))\setminus V(H)$ is adjacent (in $\delta(H)$) to a vertex in $B(T\setminus T^*)$.
	\end{itemize}
\end{definition}

Thus, we split our existing decomposition in two
parts: the \magic subtree $T^*$, where the corresponding part of the
hypergraph has changed, and the \emph{outer} subtrees which
correspond to those subhypergraphs that remain unchanged by the
modification. An important reason for considering \magic subtrees is
captured by the following Lemma~\ref{lem:keep}, namely that all the trees outside
of $T^*$ are still correct GHDs for their respective
parts of the new hypergraph. Hence, it is possible
to reuse these partial decompositions for $\delta(H)$ and save the
effort of decomposing those parts of the hypergraph again.

\begin{lemma}
	\label{lem:keep}
	Let \defGHD{} be a GHD of hypergraph $H$ with tree $T$, let $\delta$ be a modification, and let $T^*$ be a $\delta$-\magic subtree.
	For every tree $T'$ in the forest $T \setminus T^*$ it holds that \defGHDprime{} is a GHD of $\delta(H)[T']$.\footnote{Technically every edge $e$ in every edge cover $\lambda_u$ is replaced by the edge $e \cap (\bigcup_{u \in T'}B_u)$ of the induced subhypergraph.}
\end{lemma}
\begin{proof}
	Since we assume that $T'$ is a tree in the forest $T \setminus T^*$ we also have that $B(T') \subseteq B(T \setminus T^*)$. Hence, it must also hold that $H[T'] = \delta(H)[T']$ by assumption that $T^*$ is $\delta$-\magic.
	
	We now argue that \defGHDprime{} is a GHD of $H[T']$ and thus, by the previous argument, also of $\delta(H)[T']$. First, observe that the connectedness condition is clearly still satisfied in $T'$ since we never change the bags. For the covers it is clear that if $e \in E(H)$, then $e \cap B(T')$ is an edge in $H[T']$. Since $B_u \subseteq B(T')$ we clearly also have that $B_u \subseteq \bigcup_{e \in \lambda_u} e \cap B(T')$.
	What is left, is to verify that every edge $e$ of $H[T']$ is covered in $T'$. Let $e'$ be one of the edges in $H$ such that $e = e' \cap B(T')$. Since we start from a GHD of $H$, there must be a node $u$ where $e'$ is covered. Hence, all the vertices of $e$ are in $B_u$. Hence, all of the subtrees induced by the vertices in $e$ touch by the connectedness condition. Since all of the vertices of $e$ are in $B(T')$ all those subtrees must have a common node in $T'$. Hence, \defGHDprime{} is a GHD of $H[T']$ and therefore also of $\delta(H)[T']$.
\end{proof}

\begin{example}
	\label{ex:magic}
	In Example~\ref{ex:more-mods}, a $\delta \in \AddRel$ is used to create the hypergraph $H_{P_2}$ from $H_{P'}$, as in Figure~\ref{fig:mod-add-edge}. We now consider reverting this modification, i.e., the modification $\delta^{-1} \in \DelRel$ that removes the edge $\{c,i\}$, i.e., we have $\delta^{-1}(H_{P_2}) = H_{P'}$ (recall, the hypergraph of $P'$ is shown in Figure~\ref{fig:hg-cy}).
	As input for our update example, we use the width 2 GHD \defGHD{} of $H_{P_2}$ given in Figure~\ref{fig:ex.magic.decomp}.
	The two highlighted nodes in Figure~\ref{fig:ex.magic.decomp} represent a $\delta^{-1}$-\magic subtree $T^*$ of $T$.
	Observe that $T \setminus T^*$ consists of two trees that correspond to the induced subhypergraphs in Figure~\ref{fig:ex.magic.subhgs}.
	By Lemma~\ref{lem:keep}, these parts remain correct GHDs for their respective induced subhypergraphs.
	
	We could update the overall decomposition by changing the bag $\{c,e,i,k\}$ to $\{e,i,k\}$ while removing $w_7$ from the $\lambda$ label to update the decomposition to fit $P'$. Mechanically this can be checked by searching for a GHD of $\delta^{-1}(H_{P_2})[\{a,b,c,e,i,h,k\}]$ that is consistent with the surrounding trees in a certain way that will be discussed below.
\end{example}

\begin{figure}[t]
	\centering
	\subfigure[The bags of a GHD of $H'$ with width 2.
	The minimal $\delta$-\magic subtree is highlighted.]{
		\tikzstyle{htmatrix}=[%
matrix,
matrix of nodes,
nodes={draw=none},
column 1/.style={nodes={fill=gray!20,align=right},anchor=base east,minimum width=3em},
column 2/.style={anchor=base west},
ampersand replacement=\&
]

\begin{tikzpicture}[sibling distance=12em,
  every node/.style = {shape=rectangle,draw,inner sep=1,align=center}
  ]
  \node (root) [htmatrix,very thick,color=red] {\htnodenum{$a,b,c,e,h$}{$w_1,w_5$}{1}}
  child {
    node (c1) [htmatrix,very thick,color=red] {\htnodenum{$c,e,i,k$}{$w_3,w_7$}{2}} 
    child {
      node (c2) [htmatrix]{\htnodenum{$j,k,l$}{$w_6$}{4}}
    }
  }
  child {
    node (c3) [htmatrix] {\htnodenum{$a,d,f,g,h$}{$w_2,w_4$}{3}}
  };
  % \node (circ1) [shape=circle, style={color=red}, thick] at ([yshift=-6,xshift=-16]root) {\;\;\;};
  % \node (circ2) [shape=circle, style={color=red}, thick] at ([yshift=5,xshift=-10]c2) {\;\;\;};
\end{tikzpicture}
		\label{fig:ex.magic.decomp}
	}\stdhspace%
	\subfigure[Hgs induced by $T \setminus T^*$.]{
		\begin{tikzpicture}[scale=1]
	\node	(a)	at	(0,1.5)	{$a$};
	\node	(d)	at	(0,.75)	{$d$};

	\node	(f)	at	(0,0)	{$f$};
	\node	(g)	at	(1,0)	{$g$};
	\node	(h)	at	(2,0)	{$h$};

	\node	(j)	at	(4,1.5)	{$j$};
	\node	(k)	at	(4,.75)	{$k$};
	\node	(l) at	(4,0)	{$l$};
	\draw (0,1.5) ellipse [x radius=0.22, y radius=0.22];
        \draw (2,0) ellipse [x radius=0.22, y radius=0.22];
	\draw (1,0) ellipse [x radius=1.4, y radius=0.4];
        \draw (4,0.75) ellipse [x radius=0.22, y radius=0.22];
	\draw[rotate around={90:(0,.75)}] (0,.75) ellipse [x radius=1.1, y radius=0.4];

	\draw[rotate around={90:(4,.75)}] (4,.75) ellipse [x radius=1.1, y radius=0.4];
\end{tikzpicture}
		\label{fig:ex.magic.subhgs}
	}
	\caption{Example~\ref{ex:magic}.}
	\label{fig:exsubt}
\end{figure}

Since we want to reuse as much of the old decomposition as possible, it naturally becomes interesting to have $T \setminus T^*$ as large as possible. Hence, we are interested in finding \emph{minimal} $\delta$-\magic subtrees, i.e., those $\delta$-\magic subtrees with the least number of nodes. Fortunately, it is relatively easy to find minimal \magic subtrees. The full tree $T$ is trivially a $\delta$-\magic subtree. We can then start from $T_0 = T$ and greedily eliminate leaves as long as the property from Definition~\ref{def:magic} remains valid.
Once no more leaves can be removed, the procedure will have reached a minimal $\delta$-\magic subtree.\footnote{A more detailed argument is available in the technical appendix.}

\begin{lemma}
	\label{lem:alg}
	For any GHD $\mathcal{G}$ of a hypergraph $H$ and any modification $\delta$ there exists a \emph{unique} minimal $\delta$-\magic subtree.
	Moreover, there exists an algorithm with input $(\mathcal{G}, H, \delta(H))$ that computes the minimal $\delta$-\magic subtree in polynomial time.
\end{lemma}
\begin{proof}
	We first prove the uniqueness of minimal $\delta$-\magic subtrees.
	Suppose towards a contradiction that there are two distinct minimal $\delta$-\magic subtrees $T_1$ and $T_2$ of a GHD $\defGHD$.
	Recall, in the argument for Lemma~\ref{lem:keep} it was already argued that for every tree $T'$ in  $T \setminus T_1$  or $T \setminus T_2$, we have that $H[T'] = \delta(H)[T']$.
	
	Now, from the assumption that both $T_1$ and $T_2$ are minimal but distinct there has to exist a tree $T' \in T\setminus T_1$ such that $T' \cap T_2 \neq \emptyset$. If this were not true, then $T_2$ would be a subtree of $T_1$ and we are done. Fix such a $T'$ and let $X$ be the set of nodes that are in $T'$ and $T_2$.
	Since $H[T'] = \delta(H)[T']$, and $X \subseteq T'$, also $H[X] = \delta(H)[X]$ and it becomes easy to see that
	\[
	H[T \setminus (T_2 \setminus X)] = H[B(X) \cup B(T\setminus T_2)] = \delta(H)[B(X) \cup B(T\setminus T_2)] =     \delta(H)[T \setminus (T_2 \setminus X)] 
	\]
	Thus, $T_2 \setminus X$ is also a $\delta$-\magic subtree and smaller than $T_2$, contradicting our initial assumption of minimality. Note that the second condition of Definition~\ref{def:magic} cannot become unsatisfied by removing nodes from $T_2$.

	The algorithm from the statement is given in Algorithm~\ref{alg:minmagic}. The algorithm clearly starts with a $\delta$-\magic subtree and throughout the iterative elimination the \emph{working tree} $T'$ remains a $\delta$-\magic subtree. Note that the argument from before can be seen as a method to create a smaller $\delta$-\magic subtree from any disjoint pair of $\delta$-\magic subtrees.
	Hence, if no more leaves can be removed from $T'$ in the algorithm, then every smaller tree must not be disjoint. But if the minimal subtree were to be a proper subtree of $T'$, then removal of some leaf must be possible since the induced subhypergraphs of $T \setminus T'$ grow monotonically. Thus, there can be no smaller $\delta$-\magic subtree that is a subtree of $T'$ and none that has disjoint vertices from $T'$. It follows that the returned $T'$ is minimal and the algorithm is therefore correct.
\end{proof}

\begin{algorithm}[tb]
	\small
	\DontPrintSemicolon
	\SetKwInOut{KwPara}{Parameter}
	\SetKw{break}{break}\SetKw{return}{Return}
	
	\KwIn{$H$, $\delta(H)$, GHD $\defGHD{}$ of $H$}
	\KwOut{A minimal $\delta$-\magic subtree.}
	
	\SetKwProg{myalg}{Function}{}{}
	\Begin{
		$T' \coloneqq T$\;
		$V_{\mathit{new}} = V(\delta(H)) \setminus V(H)$\;
		\Repeat{$T'$ did not change}{
			\ForEach{Leaf $u$ of $T'$}{
				$T'_{-u} = T' \setminus \{u\}$\;
				$A \coloneqq H[T \setminus T'_{-u}]$\;
				$B \coloneqq \delta(H)[T \setminus T'_{-u}]$\;
				$\mathit{Adj} \coloneqq $ all vertices adjacent to $B(T \setminus T'_{-u})$ in $\delta(H)$\;
				\If{$A = B$ and $V_{\mathit{new}} \cap \mathit{Adj} = \emptyset$}{
					$T' \coloneqq  T'_{-u} $ \;
					\break \;
				}
			}
		}
		\return $T'$ \;
	}
	\caption{Finding Minimal $\delta$-\magic Subtrees.}
	\label{alg:minmagic}
\end{algorithm}

By Lemma~\ref{lem:keep} we can use the old decomposition to derive correct GHDs for certain induced subgraphs of $\delta(H)$. It is not guaranteed that the minimal width GHD of $\delta(H)$ can be constructed in such a way that these pre-solved induced subgraphs correspond to parts of the decomposition. However, the possibility of only having to recompute a decomposition for some small subgraph $\delta(H)^*$ is promising in practice.
In particular, we are interested in $\delta(H)^*$ which is the part of $\delta(H)$ that contains $\delta(H)[T^*]$ for the minimal $\delta$-\magic subtree $T^*$, plus any possible new vertices and edges introduced by $\delta$.
Ideally, we want a new GHD for $\delta(H)^*$ with which we can replace $T^* \subseteq T$ to arrive at a valid generalized hypertree for $\delta(H)$.
This way we can fully reuse the $T\setminus T^*$ parts of the old GHD.
To replace the new decomposition of $\delta(H)^*$ in
place of $T^*$ in $T$ we need to enforce some additional constraints on
the GHD of $H^*$. Therefore, we introduce the notion of
\emph{bag constraint} as a set $C \subseteq V(\delta(H)^*)$. A bag constraint $C$ is satisfied
by a GHD \defGHD{} if there exists a node $u \in T$, s.t., $C \subseteq B_u$. In particular, given such a GHD and a \magic subtree $T^*$, let $\{u_1,\dots, u_q\}$ be the set of nodes in $T$ that have a neighbor in $T^*$.
We call the set $\{C_i \mid C_i = B_{u_i} \cap (\bigcup_{u \in T^*}B_u), 1 \leq i \leq q\}$ the $T^*$\emph{-induced bag constraints}.

\begin{theorem}
	\label{thm:subinst}
	Let $\mathcal{G}$ be a width $k$ GHD of a hypergraph $H$ with tree $T$, let $\delta$ be a modification and let $T^*$ be a $\delta$-\magic subtree of $T$.
	If $\delta(H)^*$ has a GHD of width $\leq k$ that satisfies all $T^*$-induced bag constraints, then $\ghw(\delta(H)) \leq k$.
\end{theorem}
\newcommand{\starghd}{\ensuremath{\left< T', (B'_u)_{u\in T'}, (\lambda'_u)_{u\in T'} 
		\right>}}
\begin{proof}
	We prove the statement by constructing the required new width $k$
	GHD of $\delta(H)$ from the GHD of $\delta(H)^*$ and the subtrees $T \setminus T^*$. Hence, not only is the width of $\delta(H)$ at most $k$, but a GHD of $\delta(H)$ can be efficiently constructed by only computing a GHD (with bag constraints) for $\delta(H)^*$.
	
	Suppose $\calD^*$ is a width $k$ GHD of $H^*$ that satisfies all $T^*$-induced bag constraints. Let $C_1, \dots, C_\ell$ be $T^*$-induced bag constraints and recall that every bag constraint is associated one-to-one to a node in $T$ that neighbors a node in $T^*$. Let $u_i$ be the node associated to the constraint $C_i$ in this way for all $i \in [\ell]$.
	
	The final decomposition $\starghd$ is now constructed as follows starting from $\calD^*$. For each bag constraint $C_i$, identify the subtree $T_i \in T \setminus T^*$ that contains $u_i$ as well as any node $u^*_i$ in $\calD^*$ that satisfies $C_i$. Then, attach the tree $T_i$ at node $u_i$ to $\calD^*$ at $u^*_i$. By attaching subtrees for each bag constraint this way we obtain our final $\starghd$.
	
	We now argue that $\starghd$ is indeed a width $k$ GHD of
	$\delta(H)$. Indeed, width $k$ follows immediately from the
	construction since $\calD^*$ and $\defGHD{}$ both have width $k$ and
	none of their $\lambda$-labels are modified.  For connectedness,
	recall that by our definition of bag constraints the tree $T_i$ is
	attached to a node $u^*_i$ whose bag contains $B_{u_i} \cap B(T^*)$.
	Hence, every vertex in $B(T^*)$, and thus also every vertex in bags
	of $\calD^*$, that also occurs in $B(T_i)$ must be in $B_{u_i}$. We
	see that connectedness can not be violated by the attaching step of
	our construction. By Lemma~\ref{lem:keep}, all the individual parts that are attached
	to $\calD^*$ already satisfy the connectedness condition and it therefore holds also for
	all of $\starghd$.
	
	Finally, we verify that all edges of $\delta(H)$ are covered by
	some bag of $\starghd$.  We partition the set of edges in two sets,
	edges that are in $\delta(H)^*$ and those that are not. If an edge is in $\delta(H)^*$, then it must
	be covered by $\calD^*$ and thus also in $\starghd$. In the latter case, observe that if an edge $e$ is in $\delta(H)$ but not in $\delta(H^*)$, then $e$ is in $H$ and thus covered by some node of $T \setminus T^*$. Note that there is a $T^*$-induced bag constraint for every tree in $T \setminus T^*$. Hence, by Lemma~\ref{lem:keep} and the above construction reattaching the subtree in which $e$ is covered, $e$ is also covered in $\starghd$. 
	
	Note that we made no explicit use of the second condition in Definition~\ref{def:magic}. The condition effectively enforces that any edges that contain new vertices will be in $H^*$ and in this way implicitly factors into the above argument.
\end{proof}

It is possible for no $T^*$-induced bag constraints satisfying GHD of $\delta(H)^*$ with width at most $k$ to exist, even if $\ghw(\delta(H)) \leq k$. Thus, while the discussions of this section -- and in particular the ideas of Theorem~\ref{thm:subinst} -- form the foundation of our practical implementation, some adaptions are necessary to efficiently deal with those cases. This will be the topic of the following section.

\section{Algorithmic Implementation Of $\delta$-mutable Subtree Framework }
\label{sec:cem}

We now want to focus on how to use the concept of $\delta$-mutable subtrees from Section~\ref{sec:update-practice} to define an algorithmic framework that can take an existing GHD and the hypergraph of an updated CSP and provide some data structure that can be used to speed up the computation of a new GHD, even if parts of the old GHD need to be recomputed. We want this framework to be implementation agnostic, so that it can be plugged in to essentially any existing decomposition algorithm with few required adaptations. 

To this end, we first define this framework, and afterwards we also give a practical example of how it has been adapted to an existing state of the art decomposition algorithm. We then report on its performance in various update tasks in Section~\ref{sec:experiments}.

\subsection{Introducing the Framework}
\label{sub:introducing_the_framework}

The goal of our implementation is twofold: we want a strategy built on top of the framework of $\delta$-mutable subtrees,
and we want it to encompass existing algorithms for computing GHDs.
In the following, we make use of the basics of top-down GHD construction, as explained in Section~\ref{sect:ComputingHDs}.
For more detail, we refer the reader to a recent overview on hypergraph decompositions~\cite{DBLP:conf/cpaior/GottlobLLOP20}. Before we proceed with explaining our implementation, we will need some way of referring how the bags and edge covers of the old GHDs are affected by an update $\delta$ when we want to use the old GHD with the modified hypergraph $\delta(H)$. We first introduce a function $s_\delta: E(H) \rightarrow E(\delta(H)) \cup \emptyset $, which will map edges $e \in E(H)$ to their corresponding equivalent $e' \in E(\delta(H))$, if it exists, or $\emptyset$, if $\delta$ actually deleted that edge. By slight abuse of notation, given subset $X \subseteq E(H)$, we shall use $\delta(X) = \{ s_\delta(e) \mid e \in X\}$. In this same vein, we also introduce for a vertex set $Y \subseteq V(H)$, the notation $\delta(Y) = Y \cap V(\delta(H)) $. 
Another notational choice we make throughout this section is how to refer to the inputs of algorithms that deal with updated hypergraphs. Since all algorithms we present only deal with a single updated hypergraph and its subgraphs and never need the original hypergraph, we omit the use of the $\delta$ function. So instead of $\delta(H)$, we just write $H$. We will still need a $\delta$ mutable subtree, and we assume that it has been computed and is provided to the algorithms as an input. 

The idea behind our framework is to try to update the minimal $\delta$-mutable subtree $T^*$ and reuse many of the \emph{outer} subtrees. 
If this is not possible, due to the way the modification has changed the hypergraph, we still want to return a GHD of the updated CSP quickly.
Bag constraints from Theorem~\ref{thm:subinst} encode the properties necessary for parts of $T^*$ to be reused. As was mentioned, however, it is possible that in order to successfully find a new GHD of low width, we need to forgo some of them.
For our implementation we think of them as \emph{soft constraints}:  we make an effort to find GHDs that reuse $T^*$  if they exist,  and if they do not, use them as a starting point in the search space.

We realize this behavior via the concept of a \emph{scene}. 

\begin{definition}[Scene]
	Let \defGHD{} be a normal-form GHD of a hypergraph $H$. A scene mapping $\sigma \colon  2^{E(H)} \rightarrow N(T) $ is a partial mapping from a subhypergraph $H' \subseteq H$ to a node $u \in T$. The co-domain element of $\sigma$ is denoted as a \emph{scene}.
	Given a modification $\delta$ and a $\delta$-mutable subtree $T^*$,  we call $\sigma(H)$ \emph{out-scene} if $\sigma(H) \not \in T^*$ or \emph{in-scene} if $\sigma(H) \in T^*$.
\end{definition}

\newcommand{\delt}[1]{{ $\delta$(#1) }}

\begin{algorithm}[tb]
	\DontPrintSemicolon
	\SetKwFunction{algo}{DecompUpdate} 
	\SetKwFunction{getComp}{ComputeComponents}
	\SetKwFunction{checkScenery}{SceneCreation}
	\SetKwFunction{windUp}{SceneCreationUp}
	\SetKwFunction{windDown}{SceneCreationDown}
	\SetKwFunction{buildGHD}{BuildGHD}
	\SetKwFunction{decomp}{$D$}
	\SetKwProg{myalg}{function}{}{}
	
	\myalg{ \windDown{$n$: Node, $H'$: Subgraph,  $T^*$: $\delta$-mutable subtree}  } {
		$ \sigma \coloneqq \text{Create empty mapping}$ \;
		\If { \delt{$n$.Bag} $\not \subseteq   B(\delta(n\text{.Cover}))$ } {
			\textbf{return} \windUp{$n$, $H'$, $T^*$} \Comment{start upwards phase for $T_n$} \\
			\label{upCall1}
		}
		$sep \coloneqq $ \delt{$n$.Bag}  \;
		comps $\coloneqq $ $[sep]$-components of $H'$  \;
		\If { $|comps| \not = |n.Children|$ } {
			\textbf{return} \windUp{$n$, $H'$, $T^*$} \Comment{start upwards phase for $T_n$} \\
			
			\label{upCall2}
		}
		\uIf (\Comment{$n$ is on the path from root to $T^*$ or contained in $T^*$ itself}) { $\exists u \in T_{n} $ s.t. $u \in N(T^*)$ } {
			\texttt{scene} $\coloneqq$ new in-scene, mapping $H'$ to $sep$ \;
		} 
		\Else {
			\texttt{scene} $\coloneqq$ new out-scene, mapping $H'$ to $sep$
		}
		Add the mapping in \texttt{scene} to $\sigma$ \;   
		\label{line:addMap}
		\For {$u \in n$.Children} {
			\label{line:startLoop}
			\For {$c \in$ comps} {
				\If(\Comment{the component $c$ ``covered'' in $T_u$}) { $ ( V(c) \setminus \delta(n\text{.Bag}) \cap \delta(u\text{.Bag})) \not = \emptyset$  }{
					$\sigma' \coloneqq $ \windDown{$u$,$c$,$T^*$} \;     
					Add the mapping in $\sigma'$ to $\sigma$ \;        
					\textbf{break} \Comment{start next iteration of outer loop} 
				}
			}
			\textbf{return} \windUp{$n$, $H'$, $T^*$} \Comment{discard previous mappings and start upwards phase for $T_n$} \\
			
			\label{upCall3}
		}
		\label{line:endLoop}
		\textbf{return} $\sigma$ \;
	}

	\caption{SceneCreation -- Downward Phase}
	\label{alg:updateSCDown}
\end{algorithm}

\begin{algorithm}[tb]
	\DontPrintSemicolon
	\SetKwFunction{algo}{DecompUpdate} 
	\SetKwFunction{getComp}{ComputeComponents}
	\SetKwFunction{checkScenery}{SceneCreation}
	\SetKwFunction{windUp}{SceneCreationUp}
	\SetKwFunction{windDown}{SceneCreationDown}
	\SetKwFunction{buildGHD}{BuildGHD}
	\SetKwFunction{decomp}{$D$}
	\SetKwProg{myalg}{function}{}{}

	\myalg{ \windUp{$n$: Node, $H'$: Subgraph, $T^*$: $\delta$-mutable subtree }  } {
		
		$ \sigma \coloneqq \text{Create empty mapping}$ \;
		
		coveredBelow $\coloneqq \emptyset$ 
		
		\For {$u \in n$.Children} {
			\label{recStart}
			$\sigma'$, edgesCovered  $ \coloneqq$ \windUp(u, H',$T^*$) \;
			Add the mapping in $\sigma'$ to $\sigma$ \;       
			coveredBelow $= $ coveredBelow $\cup $ edgesCovered \;
		}
		\label{recEnd}
		
		coveredEdges $\coloneqq$ coveredBelow 
		
		\If(\Comment{make sure no node in $T_n$ is in $T^*$}) { $\not \exists u \in T_{n} $ s.t. $u \in N(T^*)$ } {
			
			\For { $e \in E(H')$  } {
				\label{coveredStart}
				\If {$e \subseteq$ $n$.Bag } {
					coveredEdges $=$ coveredEdges $\cup \{e\}$ 
				}
			}
			\label{coveredEnd}
			$sep \coloneqq $ $n$.Bag \Comment{we can use it as is, since we are outside of $T^*$}  \\
			comps $\coloneqq $ $[sep]$-components of coveredEdges  \Comment{coveredEdges here corresponds to  $H[T_n]$} \\
			\If { $|comps| = |n\text{.Children}| $  } {
				
				\texttt{scene} $\coloneqq$ new out-scene, mapping $H'$ to $sep$
				
				Add the mapping in \texttt{scene} to $\sigma$ \; 
				\label{addMappingUp}
			}

		}

		\textbf{return} $\sigma$, coveredEdges
	}

	\caption{SceneCreation -- Upward Phase}
	\label{alg:updateSCUp}
\end{algorithm}

Scenes avoid decomposing again parts of the hypergraph for which we already know a GHD.
Lemma~\ref{lem:keep} implies that the use of out-scenes is valid. Using in-scenes is more complex: we try to utilize them at most once to see if they help in finding a GHD of $\delta(H)$. If this leads to a reject case, we know that the scene will not be used again.
Therefore, in-scenes never harm the correctness of our approach.
We compute a scene mapping via a two-phase traversal of the old GHD and we require this old GHD to be in normal form so that we can determine which subtrees of the GHD ``cover'' certain components. This procedure can be seen in pseudocode in Algorithm~\ref{alg:updateSCDown} and Algorithm~\ref{alg:updateSCUp}, each detailing one of the two phases. We proceed to give an informal explanation below:

In the first \emph{downward phase}, the GHD is traversed top-down and the bags of the encountered nodes are used to ``replay a decomposition procedure''.
Starting at the root $u$ of the GHD, we create a new mapping $H \rightarrow u$ for the current hypergraph $H$. This can be seen in line~\ref{line:addMap}.
Then, we compute the $[B_u]$-components $C_1, \dots,C_\ell$, which we assign to the child nodes $u_1, \dots, u_\ell$ of $u$, where we have that $B_{u_i} \cap C_i \neq \emptyset$. 
Finally, we make a recursive call on each pair $(C_i, u_i)$. The recursive call and the matching from components to child nodes happens in Algorithm~\ref{alg:updateSCDown} between the lines~\ref{line:startLoop} to~\ref{line:endLoop}.
Due to the properties of GHDs, we know that each $[B_u]$-component matches with exactly one child node. 
If not, then the downward phase stops. This can only  happen  when considering nodes of $T^*$. 

The second \emph{upward phase} generates mappings in a different way.
It is called for any subtree below $T^*$, if the  downward phase stops at a non-leaf node. In the algorithm, we see this happening at lines~\ref{upCall1},~\ref{upCall2} and~\ref{upCall3}.
In this phase, instead of simulating a decomposition procedure,
we traverse the GHD in a bottom-up fashion and at every node $u$ we look
at the subtree $T_u$ to create the mapping $\{ e \in E(H) \mid e \cap (\bigcup_{u \in T} B_u) \neq \emptyset) \}
\rightarrow u$. In Algorithm~\ref{alg:updateSCUp} this is done by a recursive call on the child nodes of $u$, looking at all edges covered below them, as seen between lines~\ref{recStart} and~\ref{recEnd}. Between lines~\ref{coveredStart} to line~\ref{coveredEnd}, the algorithm computes the edges covered in the bag of $u$. Combined with the previous set, this gives us all edges that form $H[T_u]$. We are thus mapping $H[T_u]$ to $u$, as seen in line~\ref{addMappingUp}. This upward phase ensures we can make full use of Lemma~\ref{lem:keep} to consider all subtrees below $T^*$.

\begin{example}\label{ex:scenes}
	We shall consider here as our initial hypergraph $H_{P_2}$, seen in  Figure~\ref{fig:more-mods}. A GHD of $H_{P_2}$ is provided in Figure~\ref{fig:ex.magic.decomp}, we shall refer to it as $\mathcal{G}$ in the sequel.  We will use the same modification $\delta^{-1} \in \DelRel$ as introduced in Example~\ref{ex:magic}. Thus, using $\delta^{-1}(H_{P_2})$ and $\mathcal{G}$ we will create the scene mapping. 
	We start with the downward phase. 
	
	We start with the root node $u_1$ of $\mathcal{G}$ and create a scene mapping $\delta^{-1}(H_{P_2})\rightarrow u_1$. Next, we consider the $[B_{u_1}]$-components of $\delta^{-1}(H_{P_2})$, yielding components, $C_2 = \{w_3, w_6\}$ and $C_3 = \{ w_2, w_4 \}$.  We look for unique matching pairings of child nodes of $u_1$ and $[B_{u_1}]$-components. We see that $(B_{u_2} \setminus B_{u_1} )\cap V(C_2) = \{i,k\}$ and $(B_{u_3} \setminus B_{u_1} )\cap V(C_3) = \{d,f,g\}$. Since all components were matched, we proceed on the pairings ($C_2,u_2$) and  ($C_3,u_3$). 
	Next we consider the node $u_2$, and create the mapping $C_2 \rightarrow u_2$. We consider now the $[B_{u_2}]$-components of $C_2$. We get one component, $C_4 = \{ w_6\}$ and we have that $(B_{u_4} \setminus B_{u_2} )\cap V(C_4) = \{j,l\}$. Thus we proceed on the pairing ($C_4$, $u_4$). We create the mapping $C_4 \rightarrow u_4$. We note that there are no $[B_{u_4}]$-components of $C_4$, since $B_{u_4}$ already covers the entire component $C_4$. 
	We continue with ($C_3, u_3$). We create the mapping $C_3 \rightarrow u_3$. As before, we note that there are no $[B_{u_3}]$-components of $C_3$, as $B_{u_3}$ already fully covers $C_3$. Since the downward phase was never stopped at a non-leaf node, we do not proceed to the upward phase. To summarize, we get the following scene mapping: $\{  (\delta^{-1}(H_{P_2}) \rightarrow u_1 ), ( C_2 \rightarrow u_2), ( C_3 \rightarrow u_3),( C_4 \rightarrow u_4) \}$. We will see in Algorithm~\ref{alg:update} how this scene mapping can be used to speed up GHD computation under updates.
\end{example}

\begin{proposition}
	The Scene Creation algorithm, detailed in Algorithm~\ref{alg:updateSCDown}, has a  time complexity of $O(N^3)$ where $N$ is the size of the input hypergraph. 
\end{proposition}

\begin{proof}
	We first analyse the complexity of the \texttt{SceneCreationUp} function from Algorithm~\ref{alg:updateSCUp}, as it forms essentially a subroutine of the \texttt{SceneCreationDown} function. We see that \texttt{SceneCreationUp} takes as input the subtree $T_n$ for a given input node $n$. Then for each node $n' \in T_n$, it computes the induced subgraph $H[T_{n'}]$. This can be seen by the fact that it first looks at all edges that are directly covered by the bag of $n'$, and combines this set with the set of all edges covered by any descendants of $n'$ in the GHD. This operation clearly takes linear space in the size of the subgraph $H'$, which we can consider to be bounded by the size of the input graph.  Next, we compute the connected components over this induced subgraph when seperated by the bag of $n'$. Computation of connected components on undirected graphs is known to be in linear time~\cite{DBLP:journals/siamcomp/Tarjan72}\footnote{Note that this paper deals with (strongly) connected components of directed graphs, which is a strict generalization of connected components in the undirected setting. Formally, there is a trivial reduction of connected component computation of undirected graphs to strongly connected component computation of directed graphs. }. Thus the complexity is $O(N * (N +N))$ for the entire run of \texttt{SceneCreationUp}. Thus we get the quadratic runtime $O(N^2)$. 
	
	Next we look at \texttt{SceneCreationDown}. This function runs over the entire GHD and takes initially the entire graph as input. For each node $n'$ it encounters during the recursion, it computes the connected components over the current subgraph. Then it recursively proceeds over the subtree rooted at $n'$. This continues until we hit a leaf node, or one of the conditions at lines~\ref{upCall1} or \ref{upCall2} is met or if the matching of components and nodes fails at line~\ref{upCall3}, and we stop the recursion and call the function \texttt{SceneCreationUp}. Note that this can only happen once for a given subtree, ending the recursion at that point. Thus for a given node $n'$, we either have a linear operation, or call the quadratic function \texttt{SceneCreationUp}. This gives us a simple upper bound of $O(N^3)$ for \texttt{SceneCreationDown}.
\end{proof}

\paragraph{Algorithm overview.}
A pseudo-code representation of our framework can be seen in Algorithm~\ref{alg:update}, called
\texttt{GHDUpdate}. As input we expect four items:
1) the hypergraph $H$ of the updated CSP,
2) a decomposition algorithm $D$, which we call \emph{decomposer}, 
3) a GHD $G$ of the CSP before the modification, and lastly,
4) the $\delta$-mutable subtree $T^*$.
The output is a GHD of $H$ of width $\leq k$, or a reject if none can be found. 
The decomposer $D$ takes as input a subhypergraph and a scene mapping.
It produces a GHD of $H$ of width $\leq k$, or rejects if none exists. 
Our algorithm initially computes a scene mapping $\sigma$, in line~2, using the aforementioned procedure.
Then, the recursive function \texttt{DecompUpdate} is called on $H$ and $\sigma$. 
At line~5, the function checks if a scene $\sigma(H')$ exists for the current subhypergraph $H'$. This check is a stateful operation, and can change the contents of $\sigma$ in the following way: for in-scenes, it will remove them from $\sigma$ after the first time they have been checked and returned. For out-scenes, no such removal takes place. 
If a scene was reported as being defined, then at line~6, the algorithm immediately fixes the current node of the GHD with $\sigma(H')$ and avoids the use of the decomposer, which would start an expensive search for a new bag.
At line~7, we separate $H'$ into the same $[B_u]$-components we encountered while computing the old GHD.
Now we make a recursive call on each of these components in lines~8~to~10, adding each GHD produced to the set of children of $u$. We then return the thus created GHD with $u$ at its root. 
Line~12 is executed only if $H'$ has never been encountered while building the old GHD.
In this case, the decomposer $D$ is called to find a GHD of $H'$ of width~$\leq k$.

This design ensures that in either case, whether the $\delta$-mutable subtree can be simply updated, or an
entirely new GHD needs to be computed, we can use the same strategy.
Moreover, in both cases we exploit the information provided by the old GHD.
The decomposer can be any existing GHD algorithm, it just needs to be adapted to use scene mappings.

\SetAlCapFnt{\large}
\SetAlCapNameFnt{\large}

\begin{algorithm}[tb]
	\DontPrintSemicolon
	\SetKwInOut{KwPara}{Parameter}
	\SetKwInput{KwType}{Type}
	
	\KwType{Node=(Bag: Vertex Set, Edge Cover: Set of Edges, Children: Set of Nodes)}
	\KwIn{$H$: Hypergraph, $D$: Decomposer, $G$: GHD, $T^*$: $\delta$-mutable subtree}
	\KwPara{\ $k$: width parameter}
	\KwOut{a GHD of $H$ with width $\leq k$, or \textbf{Reject} if none exists}
	
	\SetKwFunction{algo}{DecompUpdate} 
	\SetKwFunction{getComp}{ComputeComponents}
	\SetKwFunction{checkScenery}{SceneCreation}
	\SetKwFunction{windUp}{SceneCreationUp}
	\SetKwFunction{windDown}{SceneCreationDown}
	\SetKwFunction{buildGHD}{BuildGHD}
	\SetKwFunction{decomp}{$D$}
	\SetKwProg{myalg}{function}{}{}
	\Begin{ $\sigma$ $ \coloneqq $ \windDown{root $r$ of G, $H$, $T^*$} \;
		\textbf{return} \algo{$H$, $\sigma$} \Comment{initial call} \\ }
	\myalg{\algo{$H'$: Subgraph, $\sigma$: Scene Mapping} }{
		\If(\Comment{$\sigma$ will report in-scenes as being defined only once}) {$\sigma(H')$ is defined} 
		{ $u \coloneqq \sigma(H')$ \Comment{Use node from $T^*$} \\
			comps $\coloneqq $ $[u.\text{Bag}]$-components of $H'$  \;
			$u$.Children $\coloneqq \emptyset$\;
			\For{$c\in$ comps} { $u$.Children $=$ $u$.Children $\cup $ \algo{$c$, $\sigma$} \; }
			
			\textbf{return} $u$  \Comment{$u$ now forms the root of the output GHD} \\
			
		}
		\textbf{return} \decomp{$H', \sigma$} \Comment{using Decomposer for this subgraph} }
	
	\caption{GHDUpdate}
	\label{alg:update}
\end{algorithm}

\subsection{Applying the Framework to an Existing Decomposition Algorithm} 
\label{sub:applying_the_framework}

To demonstrate that our framework can be applied existing combinatorial algorithms for finding GHDs, we looked at BalancedGo by~\citet{DBLP:conf/ijcai/GottlobOP20}. As it was an open-source program, we modified it to make use of the $\delta$-mutable framework. Our extension of BalancedGo that supports our proposed strategy for update handling is available at \url{https://github.com/cem-okulmus/BalancedGoUpdate}.

It is notable that the actual code for the algorithm, specifically the det-$k$-decomp algorithm, is taken verbatim from BalancedGo with only a few lines having been added. This is in addition to general functionality for extracting the mutable subtree from a given decomposition with respect to an updated graph and for computing the scene mapping. The actual use of the scene mapping inside the algorithm is a trivial process, and we believe almost any existing or future  approach that actual computes GHDs via a combinatorial process can make use of it.

\section{Empirical Evaluation}
\label{sec:experiments}

In this section, we explore the potential of updating GHDs with the methods of Section~\ref{sec:cem}.
We describe our experiments, show their results and discuss the implications of our findings.

\subsection{Methodology \& Synthetic Update Generation}
\label{sec:setup}

We compared multiple approaches for updating GHDs upon elementary modifications: \emph{Update} and \emph{Classic}.
\emph{Update} consists in our implementation of the general strategy of Section~\ref{sec:cem}
on top of the \balancedgo{} program from~\cite{DBLP:conf/ijcai/GottlobOP20}.
\emph{Classic} uses the original \balancedgo{} program to compute a GHD of the modified hypergraph from scratch.
In addition to these two we also compared the \htdleo{} program from~\cite{DBLP:conf/ijcai/SchidlerS21}. Similar to \emph{Classic}, this program computes a GHD for the modified hypergraph from scratch.

More precisely, given a hypergraph $H$ and a GHD $G$ of $H$ of width $k$,
we applied an elementary modification $\delta$ to $H$ and compared the times taken by \emph{Update}, \emph{Classic} and \htdleo{} to output a GHD of $\delta(H)$ of width $\leq k$, if it exists.
Recall that \emph{Update} first computes the minimal $\delta$-mutable subtree $T^*$ from $G$ and then tries to build a GHD of $\delta(H)$ of width $\leq k$ by reusing the parts of $G$ that were not affected by $\delta$.

We conducted our experiments on the HyperBench dataset~\cite{10.1145/3440015}.
HyperBench is a large collection of hypergraphs from applications, benchmarks, and random generation
that has been successfully used in a variety of hypergraph decomposition experiments.
By using the \localbip{} implementation of \balancedgo{} and the results of~\cite{DBLP:conf/ijcai/GottlobOP20},
we determined the optimal \ghw of 1798 out of the 2035 CSPs of HyperBench with a timeout of 1 hour per instance.
Indeed, updating a GHD of optimal width is the hardest case.
We thus used these 1798 hypergraphs, their GHDs, and their \ghw as a basis for our experiments.

For each hypergraph $H$, we randomly generated five elementary modifications per each class from Section~\ref{sec:mods} as follows.
For \UnConst{} we introduce a new vertex into $\ell$ randomly chosen edges, where $\ell$ is the average (rounded up) degree of the original hypergraph. We generate \JJoin{} modifications by merging two random vertices and \UnJoin{} modifications by splitting a vertex $x$ into two vertices $y_1, y_2$:
in half of the edges incident to $x$, we replace $x$ by $y_1$ and in the other half we replaced $x$ by $y_2$.
Notably, \AddRel{} adds an edge with average (rounded up) rank such that all vertices in the new edge
are already part of some existing edge. That means that we generate challenging cases while avoiding
the easy case where most vertices in the new edge have no effect. For \DelRel{} a random edge is removed from the hypergraph.

Note that updating a GHD of optimal width in case of $\delta \in \Const{}$ is trivial.
Indeed, let $v \in V(H)$ be the vertex removed by $\delta$ and consider a GHD $G$ of $H$ of width $k$.
A GHD of $\delta(H)$ of width $\leq k$ can be easily obtained by removing $v$ from all bags $B_u$ of $G$. In total, this process produce 44950 instances, each consisting of the original hypergraph $H$ with a known minimal GHD and a modification $\delta$.

For each such instance we compute the hypergraph $\delta(H)$.
This $\delta(H)$ is used as input for \emph{Classic} and \htdleo{} to check whether $\ghw(\delta(H)) \leq \ghw(H)$.
We also compute the minimal $\delta$-\magic subtree from the decomposition of $H$. The subtree and the decomposition, in addition to the hypergraph $\delta(H)$, were provided as input to the \emph{Update} implementation to solve the corresponding instance of \minupdateProb{}. Note that the the time to compute the $\delta$-\magic subtrees is trivial (under 1ms) for all of our instances and thus not explicitly reported.
The raw data for our experiments is provided here\footnote{ \url{https://zenodo.org/record/6481125} }.

All experiments ran on an Intel Xeon CPU E5-2650 at 2.9 GHz with 264 GB RAM.
Nevertheless, each instance used only one core of the CPU and 1 GB RAM.
We set a timeout of 30 minutes for each run, i.e., we stopped the program if this threshold was crossed.

\subsection{Results \& Discussion}

\begin{table*}[t]
	\setlength\columnsep{1pt} 
	\caption{Statistics for \emph{Classic}, \emph{Update}, and \htdleo{} shown separately for each modification.
		\emph{Mean Classic}, \emph{Mean Update}, and \emph{Mean \htdleo{}} are in milliseconds.
		All non-integer numbers were rounded to two decimal places. Timeout was set to 30 minutes.}
	\label{tab:benchmark}
	\centering
	\begin{tabular}{r|rr|rrrr|rrr}
		\toprule
		\multirow{2}{*}{Operation}	& Positive	& Better	& Mean		& Mean & Mean				& Mean		& \multicolumn{3}{c}{Timeout}	\\
		& (\%)		& (\%)		& Classic	& Update		& \htdleo{} & Speedup	& Classic	& Update& \htdleo{} \\
		\midrule
		
		\AddRel{}  & 85.07 & 81.23 & 1105.8\phantom{0}   &  27.45 & 55027.02 & 40.28 & 1269 & 757 & 3657 \\
		\DelRel{}  & 99.48 & 88.12 &  552.88 &  10.28 & 60826.44 & 53.78 &  534 & 105 & 3480 \\
		\UnJoin{}  & 95.59 & 73.2\phantom{0}   &  714.17 &  82.42 & 58177.14 &  8.66 &  776 & 386 & 3677 \\
		\JJoin{}   & 90.34 & 85.67 &  534.48 &  12.52 & 69843.39 & 42.68 &  675 & 336 & 3268 \\
		\UnConst{} & 85.76 & 65.73 & 1338.82 & 223.33 & 64166.19 &  5.99 & 1330 & 952 & 3807 \\
		
		\midrule
		
		\textbf{Total} & 91.26 & 78.8\phantom{0}  & 795.41 & 36.54 & 61239.2\phantom{0}  & 21.77 & 4584 & 2536 & 17889 \\
		
		\bottomrule
	\end{tabular}
\end{table*}

\definecolor{pastelBlue}{HTML}{7692B8 }
\definecolor{pastelGreen}{HTML}{7BDFA9}

\pgfplotsset{every tick label/.append style={font=\tiny}}
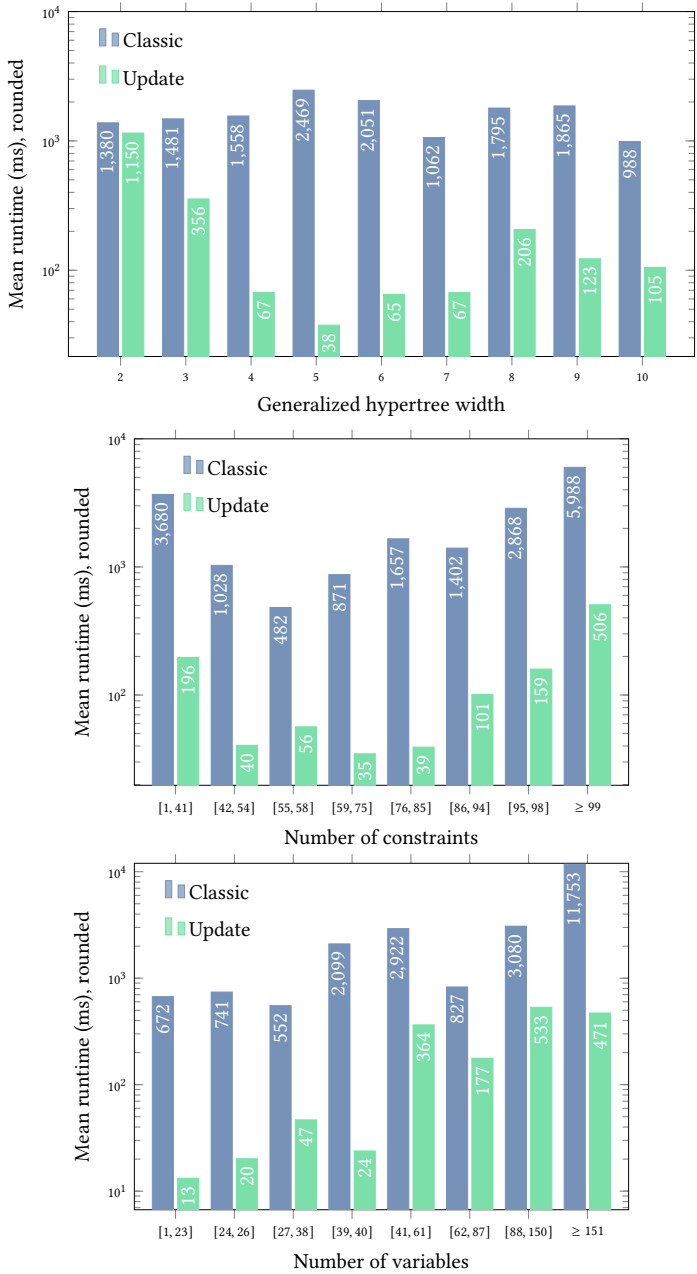
\begin{figure*}[tb]

\begin{tikzpicture} [transform shape, scale=0.8]
\begin{axis}[
        ymax=10000,
        ymode=log, ybar,  
        bar width=10pt,
        width=12cm,
        height=7.25cm,
        legend entries={Classic, Update},
        legend cell align=left,
        legend to name=grouplegend1,
        legend style={draw=none, fill opacity=0.7, text opacity = 1,row sep=5pt},       
        ylabel={Mean runtime (ms), rounded},
    xlabel={Generalized hypertree width},
        symbolic x coords={2,3,4,5,6,7,8,9,10},
        xtick=data,         
    nodes near coords  align={vertical},
       nodes near coords = \rotatebox{90}{{\pgfmathprintnumber[fixed zerofill, precision=0]
       {\pgfplotspointmeta}}},
        point meta=rawy,
    visualization depends on={meta < 5 \as \valueissmall},
    every node near coord/.append style={
            anchor={\ifdim\valueissmall  pt = 1 pt south\else north\fi},
            color={\ifdim\valueissmall  pt = 1 pt black\else white\fi}
    }
    ]
    \addplot + [pastelBlue]  coordinates {
(2,1379.57432902531)
(3,1481.2463898615672)
(4,1557.9139639731063)
(5,2468.8615134503357)
(6,2051.1537971130406)
(7,1061.5667465349316)
(8,1794.5281079490487)
(9,1864.6766750957297)
(10,988.3594741238286)
    };

    \addplot + [pastelGreen]  coordinates {
(2,1150.2326354122094)
(3,355.9549756466863)
(4,67.46601134280579)
(5,37.62323345150074)
(6,65.17301612824349)
(7,67.1355036304862)
(8,206.41829235661677)
(9,122.63006020788461)
(10,105.3194545906387)
    };
     \coordinate (leg) at (rel axis cs:0,0.75); 
\end{axis}

    \node[anchor= north west] at (leg){\pgfplotslegendfromname{grouplegend1}};

\end{tikzpicture}

\begin{tikzpicture} [transform shape, scale=0.8]
\begin{axis}[
        ymax=10000,
        ymode=log, ybar,
        legend entries={Classic, Update},
        legend cell align=left,
        legend to name=grouplegend2,
        legend style={draw=none, fill opacity=0.7, text opacity = 1,row sep=5pt},       
        ylabel={Mean runtime (ms), rounded},
        xlabel={Number of constraints},
        symbolic x coords=  {
        {$ [ 1 \comma 41 ] $},
        {$ [ 42 \comma 54 ] $},
        {$ [ 55 \comma 58 ] $},
        {$ [ 59 \comma 75 ] $},
        {$ [ 76 \comma 85 ] $},
        {$ [ 86 \comma 94 ] $},
        {$ [ 95 \comma 98 ] $},
        {$ \geq 99 $}
        },
        point meta=explicit symbolic,
        xtick=data, 
                nodes near coords,
    nodes near coords  align={vertical},
       nodes near coords = \rotatebox{90}{{\pgfmathprintnumber[fixed zerofill, precision=0]
       {\pgfplotspointmeta}}},
        point meta=rawy,
    visualization depends on={meta < 5 \as \valueissmall},
    every node near coord/.append style={
            anchor={\ifdim\valueissmall  pt = 1 pt south\else north\fi},
            color={\ifdim\valueissmall  pt = 1 pt black\else white\fi}
    },
    x post scale=1.2
    ]

    \addplot + [pastelBlue]  coordinates {    
({$ [ 1 \comma 41 ] $},3680.1244076933954)
({$ [ 42 \comma 54 ] $},1028.3625634075563)
({$ [ 55 \comma 58 ] $},482.3369649778562)
({$ [ 59 \comma 75 ] $},870.7453346856847)
({$ [ 76 \comma 85 ] $},1657.1071526629783)
({$ [ 86 \comma 94 ] $},1402.26290113947)
({$ [ 95 \comma 98 ] $},2868.441063534806)
({$ \geq 99 $},5988.299723956984)
    };
    \addplot + [pastelGreen]  coordinates {
({$ [ 1 \comma 41 ] $},196.49923969526645)
({$ [ 42 \comma 54 ] $},40.4543970970014)
({$ [ 55 \comma 58 ] $},56.4370927068148)
({$ [ 59 \comma 75 ] $},34.75891036666831)
({$ [ 76 \comma 85 ] $},39.16462908646663)
({$ [ 86 \comma 94 ] $},101.21157912587255)
({$ [ 95 \comma 98 ] $},159.37897359776625)
({$ \geq 99 $},506.14804946610576)
    }; 
     \coordinate (leg) at (rel axis cs:0.03,0.75); 
\end{axis}

    \node[anchor= north west] at (leg){\pgfplotslegendfromname{grouplegend2}}; 
\end{tikzpicture}
\begin{tikzpicture} [transform shape, scale=0.8]
\begin{axis}[
        ymax=12000,
        ymode=log, ybar,
        legend entries={Classic, Update},
        legend cell align=left,
        legend to name=grouplegend3,
        legend style={draw=none, fill opacity=0.7, text opacity = 1,row sep=5pt},       
        ylabel={Mean runtime (ms), rounded},
        xlabel={Number of variables},
        symbolic x coords=  {
        {$ [ 1 \comma 23 ] $},
        {$ [ 24 \comma 26 ] $},
        {$ [ 27 \comma 38 ] $},
        {$ [ 39 \comma 40 ] $},
        {$ [ 41 \comma 61 ] $},
        {$ [ 62 \comma 87 ] $},
        {$ [ 88 \comma 150 ] $},
        {$ \geq 151 $}
        },
        point meta=explicit symbolic,
        xtick=data, 
                nodes near coords,
    nodes near coords  align={vertical},
       nodes near coords = \rotatebox{90}{{\pgfmathprintnumber[fixed zerofill, precision=0]
       {\pgfplotspointmeta}}},
        point meta=rawy,
    visualization depends on={meta < 5 \as \valueissmall},
    every node near coord/.append style={
            anchor={\ifdim\valueissmall  pt = 1 pt south\else north\fi},
            color={\ifdim\valueissmall  pt = 1 pt black\else white\fi}
    },
    x post scale=1.2
    ]

    \addplot + [pastelBlue]  coordinates {    
({$ [ 1 \comma 23 ] $},671.7205979665981)
({$ [ 24 \comma 26 ] $},741.2125315885274)
({$ [ 27 \comma 38 ] $},552.1159777933979)
({$ [ 39 \comma 40 ] $},2099.4176640925984)
({$ [ 41 \comma 61 ] $},2921.6048299974505)
({$ [ 62 \comma 87 ] $},827.1860816524298)
({$ [ 88 \comma 150 ] $},3079.8842997861457)
({$ \geq 151 $},11753.33254210318)
    };

    \addplot + [pastelGreen]  coordinates {
({$ [ 1 \comma 23 ] $},13.223968949096703)
({$ [ 24 \comma 26 ] $},20.16782173754198)
({$ [ 27 \comma 38 ] $},46.71889310429051)
({$ [ 39 \comma 40 ] $},23.857279342735875)
({$ [ 41 \comma 61 ] $},364.3403531240072)
({$ [ 62 \comma 87 ] $},177.19510392112198)
({$ [ 88 \comma 150 ] $},533.1911348572197)
({$ \geq 151 $},470.86851373322594)
    }; 
     \coordinate (leg) at (rel axis cs:0,0.75); 
\end{axis}

    \node[anchor= north west] at (leg){\pgfplotslegendfromname{grouplegend3}}; 

\end{tikzpicture}

\caption{Geometric mean runtimes (log. scale) of \emph{Classic} and \emph{Update} w.r.t. \ghw and instance size.}
\label{speedup}
\end{figure*}

To reduce the effect of variance, we only report on the 26013 instances for which it took \emph{Classic} more than 15 milliseconds to compute a decomposition.
In the ``easier'' cases, it is reasonable to just use \emph{Classic} instead of the more sophisticated \emph{Update}.
If we move the threshold to any $t>15$, the superiority of \emph{Update} becomes even clearer.
This suggests that our approach is even more fruitful when applied to ``hard'' cases.

Since our \emph{Update} approach is built on top of \emph{Classic}, we will use only the latter as a baseline for our experiments. As it will be evident, this is also justified by the fact that \emph{Classic} performs better than \htdleo{} on average.

The results for each class of modification are shown in Table~\ref{tab:benchmark}.
The column \emph{Positive} contains the percentage of cases where the width of the hypergraph did not increase due to the modification.
The column \emph{Better} contains the percentage of instances in which \emph{Update} outperformed \emph{Classic}.
In the next three columns we record the geometric means (in milliseconds) for \emph{Classic}, \emph{Update}, and \htdleo{}.
We then report on the speedup, which is defined as the ratio between \emph{Classic} and \emph{Update} runtimes, via the geometric mean of all speedups.
In the last columns, we compare the number of \emph{exclusive timeouts} for each solver. For instance, the column \emph{Classic} reports on the number of instances that timed out for \emph{Classic}, but neither for \emph{Update} nor \htdleo{}.
Finally, for each operation, we show the number of instances that timed out for all methods.
Since computing $T^*$ takes far less than a ms for all of our instances, the time is not reported explicitly.

In order to compare the different approaches, we adopt the same methodology that was adopted in~\cite{10.1145/3440015}, i.e., we compare mean running times and number of instances that timed out.
Overall, Table~\ref{tab:benchmark} clearly demonstrates the significant benefits of using \emph{Update}.
For every modification class, the \emph{Update} mean time is significantly lower than the other approaches.
The mean speedups are very high throughout all modifications even in the most difficult cases, i.e., \UnJoin{} and \UnConst{}.
We also see that \htdleo{} seems to have a very hard time with most of the test instances, and has by far the most timeouts and the larger mean times in comparison with the other two methods.

Interestingly, \emph{Update} seems to be particularly well suited for \DelRel{} and \JJoin{} modifications.
In theory, \DelRel{} is problematic since the deleted edge could have covered an arbitrarily complex structure.
However, it seems that this occurs rarely in practice and deleting an edge simplifies the hypergraph instead.
This is clearly apparent in the observation that 99.48\% of \DelRel{} instances were positive, i.e., the width did not increase by deleting a constraint.

The \emph{Better} column shows that \emph{Update} is faster than \emph{Classic} in 78.8\% of cases on average.
This is despite the fact that many instances were solvable by \emph{Classic} in less than 40 milliseconds
and \emph{Update} has an additional overhead because of the scene mapping creation.
Another source of slowdowns are negative instances ($\ghw(\delta(H)) > k$), where the entire search space needs to be explored.
In this case, the scene mapping is of little use and its creation only causes delays.
Moreover, the \emph{Timeout} columns show that \emph{Update} solves $\approx 94\%$ of the instances, while \emph{Classic} and \htdleo{} solve $89.2\%$ and $60.2\%$ of them, respectively.

The \emph{Positive} column shows that elementary modifications do not change the width of the hypergraph in 91.26\% of cases.
This suggests that the classes we defined are indeed natural and simple:
they capture small hypergraph modifications that do not increase the complexity of the related CSP.

We also investigated how our approach behaves with increasing $\ghw$ of the input decomposition as well as in relation to hypergraph size (in number of constraints and vertices, separately). The results of both studies are summarized in Figure~\ref{speedup}.
Note that the runtimes are given on a logarithmic scale. Since \htdleo{} is more than one order of magnitude slower than the other two methods, we do not report on it.
We see that beginning from width 3, \emph{Update} provides significantly better mean
runtimes than \emph{Classic}, and the speedup generally increases as well.

We observe that the superiority of \emph{Update} becomes more pronounced as the input
CSPs (hypergraphs) become larger. Intuitively this is explained by the fact that the modification
usually affects a smaller fraction of the hypergraph as the size increases. Hence, if it is possible
to replace the mutable subtree and reuse much of the old decomposition, as shown in Section~\ref{sec:cem},
then the strengths of \emph{Update} are emphasized. In practice this is
particularly promising since recomputation of a GHD is problematic particularly for larger
instances.

\section{Conclusion and Future Work}
\label{sec:conclusion}

We introduced the problem of updating GHDs under modifications to the underlying CSP. 
We presented theoretical foundations -- in the form of $\delta$-\magic subtrees -- as well as concrete algorithmic strategies viable for all modifications and top-down decomposition algorithms.
We implemented these strategies on top of an existing competitive implementation from the literature, allowing it to make use of old decompositions to improve computation of GHDs in the update scenario.
Though experimental evaluation we verified that our approach, on average, greatly speeds up the computation of GHDs in response to elementary modifications.

This paper represents only a first step into this new challenging problem and much is left to be done.
For an immediate next step, we are particularly interested in how the input decomposition affects update performance.
Particular decompositions, e.g., \emph{balanced separator decompositions}~\cite{10.1145/3440015}, may affect the shape and size of \magic subtrees and therefore also affect how easily parts of a tree can be reused. 
Furthermore, we see potential in identifying specific modifications where \minupdateProb{}, or some variation of it, can in fact be solved efficiently. Despite our negative results for elementary modifications, it may be of interest to identify relevant special cases of the elementary modifications that allow for easier updates.

\bibliographystyle{ACM-Reference-Format}
\bibliography{ghd-update}

\appendix

\section{Proof of Theorem~\ref{thm:hard}}
\label{sec:newproof}

The argument will require details of the reduction of \textsc{3-Sat} to checking whether a hypergraph has \ghw at most $2$ by~\citet{DBLP:journals/jacm/GottlobLPR21}. The reduction is highly technical and we recall the construction and key facts here for convenience. For full details we refer to~\cite{DBLP:journals/jacm/GottlobLPR21}. It will be convenient to use $[n]$ for integer $n$ to refer to the set $\{1,2,\dots,n\}$.

\subsection{Reducing \textsc{3-Sat} to Checking $\ghw \leq 2$}
The hypergraph $H$ to be constructed consists of 3 main parts: two versions of a gadget introduced below and a subhypergraph encoding the clauses of the \textsc{3-Sat} instance. We first fix some notation. We write $[n]$ for the set $\{1,\dots,n\}$. Extending this common notation, we write $[n;m]$ for the set of pairs $[n]\times [m]$. Furthermore, we refer to the element $(1,1)$ of any set $[n;m]$ as $\min$ and $(n,m)$ as $\max$.

For two disjoint sets $M_1, M_2$ and $M = M_1 \cup M_2$ the construction makes use of a gadget with vertices $V=\{a_1,a_2,b_1,b_2,c_1,c_2,d_1,d_2\} \cup M$ and edges $E_A \cup E_B \cup E_C$ as follows:
\begin{align*}
	E_A = \{ &\{a_1,b_1\} \cup M_1, \{ a_2, b_2 \} \cup M_2, 
	\{a_1,b_2\},\{a_2,b_1\}, \{a_1, a_2\} \} \\ 
	E_B = \{ &\{b_1,c_1\} \cup M_1, \{ b_2, c_2 \} \cup M_2, \{b_1,c_2\},
	\{b_2,c_1\}, \{ b_1,b_2\}, \{c_1,c_2\} \} \\ 
	E_C = \{ & \{c_1,d_1\} \cup M_1, \{ c_2, d_2 \} \cup M_2, 
	\{c_1,d_2\},
	\{c_2,d_1\}, \{ d_1,d_2\} \}  
\end{align*}

Let $\varphi = \bigwedge_{j=1}^m 
(L_j^1 
\vee L_j^2 \vee L_j^3)$ be an arbitrary instance of  \textsc{3-Sat} with $m$ clauses 
and 
variables 
$x_1,\ldots,x_n$. 
In addition to the vertices for two of the aforementioned gadgets, the reduction uses the following sets to construct the target hypergraph $H$:
\begin{description}
	\item[$Y, Y', Y_\ell, Y'_\ell$:] The sets $Y = \{ y_1, \ldots, y_n\}$ and
	$Y' = \{ y'_1, \ldots, y'_n\}$ will encode the truth values of the
	variables of $\varphi$. $Y_{\ell}$ ($Y'_{\ell}$) are the sets
	$Y \setminus \{y_{\ell}\}$ ($Y' \setminus \{ y'_{\ell}
	\}$).
	\item[$A, A', A_p, A'_p$:]
	We have sets
	$A = \{ a_p \mid p \in [2n+3; m] \}$ and
	$A' = \{ a'_p \mid p \in [2n+3;m]\}$ with the following important subsets:
	\begin{align*}
		A_p &= \{ a_{\min},\ldots,a_p \} & 
		\overbar{A_p} &=  \{ a_{p},\ldots,a_{\max} \} \\
		A'_p &= \{ a'_{\min},\ldots,a'_p \} & 
		\overbar{A'_p} &= \{ a'_{p},\ldots,a'_{\max} \} 
	\end{align*}
	\item[$S$:] First define $Q = [2n+3;m] \cup
	\{(0,1),(0,0),(1,0)\}$. Then, $S$ is defined as
	$Q \times \{1,2,3\}$. The elements in $S$ are pairs, which we denote
	as $(q \mid k)$. The values $q \in Q$ are themselves pairs of
	integers $(i,j)$.
	\item[$S_p$:]
	For $p \in [2n+m;m]$ we write $S_p$ for the set $\{ (p, 1), (p,2), (p,3) \}$. And $S^k_p$ for the singleton $\{ (p \mid k) \}$ for $k \in \{1,2,3\}$.
\end{description}

The vertices of $H$ are as follows.
\begin{align*}
	V(H) = &\; S \;\cup\; A \;\cup\; A' \;\cup\;  Y \;\cup\; Y' 
	\;\cup\; \{ 
	z_1, 
	z_2 \} \;\cup \\
	&\; \{ a_1, a_2, b_1, b_2, c_1, c_2, d_1, d_2,
	a'_1, a'_2, b'_1, b'_2, c'_1, c'_2, d'_1, d'_2 \}.
\end{align*}

The edges of $H$ are defined below. First, we take two copies of the
gadget $H_0$ described above:
\begin{itemize}
	\item Let $H_0 = (V_0, E_0)$ be the hypergraph of the lemma described at the beginning of the section
	with $V_0=\{a_1,a_2,b_1,b_2$, $c_1,c_2,d_1,d_2\} \cup M_1 \cup M_2$ and 
	$E_0 = E_A \cup E_B \cup E_C$, where we set $M_1 = S \setminus 
	S_{(0,1)} \cup \{z_1\}$ and $M_2 = Y \cup 
	S_{(0,1)} \cup \{z_2\}$.
	\item Let $H'_0 = (V'_0, E'_0)$ be the corresponding hypergraph, 
	with $V'_0=\{a'_1, a'_2, b'_1,$ $b'_2, 
	c'_1, c'_2, d'_1, d'_2\}\cup M'_1 \cup M'_2$ and $E'_A, E'_B, E'_C$ 
	are the primed versions of the 
	edge sets $M'_1 = S 
	\setminus S_{(1,0)} \cup \{z_1\}$ and $M'_2 = Y' \cup S_{(1,0)} \cup  
	\{z_2\}$.
\end{itemize}
Beyond the gadget $H$ contains the following edges.
\begin{itemize}
	\item $e_{p} = A'_p \cup \overbar{A_p}$,
	for $p \in [2n+3;m]^-$,
	\item $e_{y_i} = \{ y_i, y'_i \}$, for $1 \leq i \leq n$,
	\item For $p = (i,j)  \in [2n+3;m]^-$ and $k \in \{1,2,3\}$:
	\begin{align*}
		e^{k,0}_p = & \begin{cases}
			\overbar{A_p} \cup (S\setminus S^{k}_p) 
			\cup Y  \cup \{z_1\} & \mbox{if } L^k_j = x_{\ell} \\
			\overbar{A_p} \cup (S\setminus S^{k}_p)
			\cup Y_{\ell} \cup \{z_1\} & \mbox{if } 
			L^k_j = \neg x_{\ell}, 
		\end{cases} \\
		e^{k,1}_p = & \begin{cases}
			A'_p \cup S^{k}_p \cup 
			Y'_{\ell} \cup \{ z_2\} & \mbox{if } L^k_j = x_{\ell} \\
			A'_p \cup S^{k}_p \cup 
			Y' \cup \{z_2\} & \mbox{if } L^k_j = \neg x_{\ell}. 
		\end{cases}        
	\end{align*}
	\item $e^0_{(0,0)}=
	\{ a_1 \} \cup A \cup S \setminus S_{(0,0)} \cup Y \cup \{ z_1\}
	$
	\item $e^1_{(0,0)} = S_{(0,0)} \cup Y' \cup \{ z_2\}$
	\item $e^0_{\max} =
	S \setminus S_{\max} \cup Y
	\cup \{ z_1\}$
	\item $e^1_{\max} = \{a'_1 \} \cup A' \cup  S_{\max} \cup 
	Y' \cup \{z_2\}$
\end{itemize}

\paragraph{The key GHD}
The hypergraph construction above is such that only certain (if any)
width 2 GHDs are possible for $H_\varphi$. In particular, in~\cite{DBLP:journals/jacm/GottlobLPR21} it is shown extensively that any width 2 GHD of $H_\varphi$ needs to be a line, i.e., it has no branching. Furthermore, the gadget construction is used to specify two blocks of nodes that need to be at the two ends of the line, indirectly fixing the possible nodes between them. Here it is enough to consider the standard GHD that can be constructed when a satisfying assignment $\sigma$ for $\varphi$ is known. Recall, $x_1, \dots, x_n$ are the variables of $\varphi$ and let
\[
Z = \{ y_i \in Y \mid \sigma(x_i)=1 \} \cup \{ y'_i \in Y' \mid \sigma(x_i) = 0 \}.
\]
The following line graph, together with $B_u$ and $\lambda_u$ labels from Table~\ref{tab:np_decomp}, describe a width 2 GHD $\mathcal{D}_\varphi$ for $\varphi$.
\[
u_C - u_B - u_A - u_{\min\ominus1} - u_{(1,1)} - \cdots - u_{(2n+3,m-1)} - u_{\max} - u'_A - u'_B - u'_C
\]
In the following arguments we will make use of this basic structure to argue the existence of width 2 GHDs for other cases.
\begin{table*}
	\caption{Definition of $B_u$ and $\lambda_u$ for GHD of $H$.}
	\label{tab:np_decomp}
	\centering
	\begin{tabular}{|c|c|c|}
		\hline
		$u \in T$ & $B_u$ & $\lambda_u$ \\ 
		\hline 
		$u_C$ & $\{d_1, d_2, c_1, c_2\} \cup Y \cup S \cup \{z_1,z_2\}$ & 
		$\{c_1,d_1\}\cup M_1$, 
		$\{c_2,d_2\}\cup M_2$   \\
		$u_B$ & $\{c_1, c_2, b_1, b_2\} \cup Y \cup S \cup \{z_1,z_2\}$ & 
		$\{b_1,c_1\}\cup M_1$, 
		$\{b_2,c_2\}\cup M_2$  \\
		$u_A$ & $\{b_1, b_2, a_1, a_2\} \cup Y \cup S \cup \{z_1,z_2\}$ & 
		$\{a_1,b_1\}\cup M_1$, 
		$\{a_2,b_2\}\cup M_2$  \\
		$u_{\min\ominus 1}$ & $\{a_1 \} \cup A \cup Y \cup S \cup Z 
		\cup \{z_1,z_2\}$ &
		$e^0_{(0,0)}, e^1_{(0,0)}$  \\
		$u_{p \in [2n+3;m]^-}$ & $A'_p \cup \overbar{A_p} \cup 
		S \cup Z \cup \{z_1,z_2\}$ & $e^{k_p,0}_p, e^{k_p,1}_p$  \\
		$u_{\max}$ & $\{ a'_1\} \cup A' \cup Y' \cup S \cup Z \cup \{z_1,z_2\}$ & 
		$e^0_{\max}, e^1_{\max}$  \\
		$u'_A$ & $\{a'_1, a'_2, b'_1, b'_2\} \cup Y' \cup S \cup \{z_1,z_2\}$ & 
		$\{a'_1,b'_1\} \cup 
		M_1'$, $\{a'_2,b'_2\}\cup M'_2$  
		\\
		$u'_B$ & $\{b'_1, b'_2, c'_1, c'_2\} \cup Y' \cup S \cup \{z_1,z_2\}$ & 
		$\{b'_1,c'_1\} \cup 
		M_1'$, $\{b'_2,c'_2\}\cup M'_2$  
		\\
		$u'_C$ & $\{c'_1, c'_2, d'_1, d'_2\} \cup Y' \cup S \cup \{z_1,z_2\}$ & 
		$\{c'_1,d'_1\} \cup 
		M_1'$, $\{c'_2,d'_2\}\cup M'_2$  
		\\
		\hline
	\end{tabular}
\end{table*}

\subsection{Adapting the Argument to Updates}
We now show how to use the construction for the reduction from \textsc{3-Sat} to checking $\ghw \leq 2$ to the update problem restricted to the stated classes of atomic updates.
Recall that for an instance $\varphi$ of \textsc{3-Sat}, the constructed hypergraph $H_\varphi$ has $\ghw(H_\varphi)=2$ if $\varphi$ is satisfiable, and width 3 otherwise. Our plan is to manipulate $H_\varphi$  in such a way that we can efficiently construct a width 2 GHD of $H' = \delta^{-1}(H_\varphi)$ (and $H'$ is not acyclic) for $\delta$ in the respective classes.
If such a modification $\delta \in \Delta$ always exists, then the satisfiability of $\varphi$ many-one reduces to the decision version of \minupdateProb{}$(\Delta)$ for inputs $H'$, $\delta$, and the width 2 GHD of $H'$. We are able to give such a reduction for the \DelRel{} and \UnJoin{} case. We will describe below how to handle the other operations using a slightly more involved strategy.

\textbf{\DelRel{}}.
Here our goal is easy to reach. To obtain $H'$ it is sufficient to add a large edge $e^* = V(H)\setminus \{d_1\}$ to $H_\varphi$ is sufficient. Since $d_1$ has at least two distinct edges to other vertices (which are in $e^*$) we see that the resulting $H'$ is not acyclic. Clearly then $\ghw(H')=2$ and it is trivial to construct an appropriate width 2 GHD.

\textbf{\UnJoin{}}.
Let $H'$ be the hypergraph performing a \JJoin modification on vertices $z_2$ and $a'_1$ in $H_\varphi$, (using $a'_1$ to represent the vertex after the join). In particular this will merge edges containing $M'_1$, $M_2$ and $M'_2$ in the two gadgets as well as all edges of form $e_p^{k,1}$ (as well as some linking edges in the gadget). The resulting edge $e^*$ is of the form
\[
e^* = S \cup  Y \cup Y' \cup A' \cup \{a_1', a'_2, b'_1, b'_2, c'_2, d'_2, b_2, c_2, d_2, z_1\}
\]
The GHD $\mathcal{D}_\varphi$ given above can then be adapted in the following manner to yield a GHD of width 2 for $H'$. Replace all edges that contained $z_2$ or $a'_1$ in covers $\lambda_u$ by the new $e^*$. Note that this affects all nodes in the GHD. Then add $e^*$ to the bag of every node.
Clearly, all edges of $H_\varphi$ are still covered and our new edge $e^*$ is covered in every node. Finally, $H'$ is not acyclic as some cycles in the gadgets remain untouched by the merge. For example, the edges $\{c_1, d_2\}, \{c_1, c_2\}$ form an $\alpha$-cycle with $e^*$ (note that $e^*$ does not contain $c_1$).

\subsection{The Complex Cases -- \UnConst{}, \AddRel{}, and \JJoin{}}
For the other modification classes we will now slightly change our strategy and instead show how to decide satisfiability of \textsc{3-Sat} via a polynomial number of calls to \minupdateProb{}. Note that we use the returned new decompositions from the calls and thus can not directly derive \np-hardness by Turing reduction of the \minupdateProb{} decision problem.
Formally, for a class $\mathcal{C}$ of updates, instead of a single $H'$ we construct a sequence $H'_0,\dots,H'_\ell$ with $H'_\ell = H_\varphi$, $\ell$ polynomial in the size of $\varphi$, and for each $i \in [\ell]$ there is a $\delta_i\ \in \mathcal{C}$ such that $\delta_i(H'_{i-1}) = H'_i$. We will show that for all $0 \leq i \leq \ell$, $\ghw(H_i) \leq 2$ if and only if $\varphi$ is satisfiable.

Suppose now that we can construct such $H'_0$, sequence of modifications $\delta_1,\delta_2,\dots,\delta_\ell$, as well as a width 2 GHD for $H'_0$ efficiently. If \minupdateProb{}($\mathcal{C}$) were feasible in polynomial time, then we could verify $\ghw(H_\varphi) \leq 2$ in polynomial time by iteratively constructing a GHD for it from successive calls to \minupdateProb{}, starting from the known GHD of $H'_0$. As stated previously, $\varphi$ is satisfiable if and only if $\ghw(H_\varphi)\leq 2$ and thus this would yield a polynomial procedure for solving \textsc{3-Sat}.

\textbf{\UnConst{}}. 
The desired sequence of hypergraphs and modifications is defined via  $\delta_i^{-1}$ being the modification that removes vertex $y'_i$ from $H'_i$. Thus, $\ell=n$ and all $\delta_i \in \UnConst{}$.
Observe that \UnConst{} modifications can never decrease \ghw, that is $\ghw(H'_{i-1}) \leq \ghw(H'_i)$ for all $i \in [\ell]$. This can be easily observed since their inverse, the modifications of \Const{}, produce an induced subhypergraph and thus can not increase \ghw. Thus, we see that if $\ghw(H'_0)=2$ and $\ghw(H'_\ell)=2$ (which is equivalent to $\varphi$ being satisfiable), then $\ghw(H'_\ell)=2$ for all $1 \leq i \leq \ell$.  We therefore see that this gives us a sequence of modifications as described above.

Recall that in the original GHD $\mathcal{D}_\varphi$ given above, the set $Z$ is derived from a satisfying assignment for $\varphi$. For $H'_0$ we can simply set $Z=\emptyset$ (and remove all $y'_i$ from the bags) to obtain a width 2 GHD. Only the edges $e_{y_i}$ relied on $Z$ to be covered in $H_\varphi$, but in $H'_0$ they are all singletons $\{y_i\}$ and thus always covered in the first gadget. Hence, we can construct a width 2 GHD for $H'_0$ (which is cyclic) and as described above, a linear number of calls to \minupdateProb{}($\UnConst{}$) are sufficient to decide whether $\varphi$ is satisfiable.

\textbf{\AddRel{}}.
We again argue via a sequence
$H'_0, \dots, H'_\ell$ of hypergraphs and modifications
$\delta_1,\dots,\delta_\ell \in \AddRel{}$ with
$\delta_i(H'_{i-1})=H'_i$.
In contrast to the \UnConst{} case, \AddRel{} modifications can decrease \ghw and our use of such a sequence thus depends on particular properties of our choice of modification sequence.

We define our sequences via $\delta^{-1}_i$ being the modification (in \DelRel{}) that deletes edge $e_{y_i}$ from $H_\varphi$. The construction of a width 2 GHD for $H'_0$ is the same as for $\mathcal{D}_\varphi$ above but with $Z= \emptyset$. The function of $Z$ is only to connect $u_{\min\ominus1}$ and $u_{\max}$ in a way such that every $e_{y_i}$ is covered in either one of the respective bags. Since $H'_0$ no longer contains those edges, this is still satisfied with $Z=\emptyset$. It is not difficult to verify that $Z$ was not used to cover any other edges in $H_\varphi$ and therefore the correctness of the resulting GHD. Now, suppose that $\ghw(H_\varphi)=2$, then there exists some width 2 GHD $\mathcal{D}_\varphi$ of the form shown above. Note that no $e_{y_i}$ edge is used in a $\lambda_u$ set for this GHD but all of them are covered in some bag. In consequence, $\mathcal{D}_\varphi$ is also a GHD for every hypergraph $H'_i$ in our sequence, meaning every hypergraph in the sequence has \ghw 2 iff $\ghw(H_\varphi)=2$. Note that while it is hard to find $\mathcal{D}_\varphi$, the key point here is that a width 2 GHD for the special case $H'_0$ can always be found easily. 

We can now proceed as in the \UnConst{} case. Start from input $H'_0$, $\delta_1$, and the width 2 GHD of $H'_0$ as described above and call \minupdateProb{}($\AddRel{}$) to find a width 2 GHD for $H'_1$. Iterating this process, we either arrive at some $H'_i$ for which $\ghw(H'_i)>2$ and reject or we show that $\ghw(H_\varphi)\leq 2$. In the former case, we have by the argument above that then also $\ghw(H_\varphi)>2$. Thus, we can correctly decide whether $\ghw(H_\varphi)\leq 2$ -- and therefore also \textsc{3-Sat} -- using linearly many calls of \minupdateProb{}($\AddRel{}$).

\textbf{\JJoin{}} We construct the initial hypergraph $H'_0$ from
$H_\varphi$, by replacing every edge $e_{y_i} = \{y_i, y'_i\}$ by
the edge $e^*_i = \{y_i, \star_i\}$. Consider the sequence
$\delta_1,\dots,\delta_n$ such that $\delta_i \in \JJoin{}$ merges
$\star_i$ into $y'_i$, i.e., $\star_i$ is replaced in every edge by
$y'_i$.  It is easy to see that $H'_n = H_\varphi$ and if
$\ghw(H'_i) = 2$ for all $i\in[n]$, then $\ghw(H_\varphi)=2$. We will
first argue that $H'_0$ has \ghw 2 and that witnessing GHD can be
found easily. Then we show that if $\ghw(H_\varphi)=2$, then
$\ghw(H'_i)=2$ for all $i\in [n]$. All together this again means that
it is possible to decide \textsc{3-Sat} using a linear number of
calls to \minupdateProb{}($\JJoin{}$).

The decomposition for $H'_0$ is again based on $\mathcal{D}_\varphi$ with $Z=\emptyset$. Observe that $\mathcal{D}_\varphi$ does not use any $e_{y_i}$ as a cover and thus the only concern with adapting it for $H'_0$ is making sure that every $e^*_i$ is covered in some bag. To that end, add nodes $u^*_i$, for $i \in[n]$ as children of $u_{\min\ominus 1}$ with $B_{u^*_i}= e^*_i$ and cover $\lambda_{u^*_i} = \{e^*_i\}$. The connectedness condition is clearly not violated by these new nodes and every $e^*_i$ is now covered. Let $\mathcal{D}_0$ be the GHD described here and note it clearly has width 2 (and $H'_0$ is not acyclic).

To see that every $H'_i$ for $1 \leq i \leq n$ has $\ghw(H'_i)=2$,
if $\ghw(H_\varphi)=2$, we now proceed in similar fashion to the
argument for $\AddRel{}$. Since we assume that $\ghw(H_\varphi)=2$, there exists a satisfying assignment $\sigma$ for the variables of $\varphi$.
Let $Z=\{y_i \mid \sigma(x_i)=1\} \cup \{y'_i \mid \sigma(x_i)=0\}$ be the set as in the original definition of $\mathcal{D}_\varphi$.
Let $\mathcal{D}_i$ be the GHD obtained from $\mathcal{D}_\varphi$ using this $Z$ and for all $j$ s.t. $i < j \leq n$, add nodes $u^*_j$ as children of $u_{\min \ominus 1}$ as in the construction of $\mathcal{D}_0$ above. By construction, $H'_i$ contains the edges $e_{y_j}$  for $j\leq i$ and edges $e^*_j$ for $j > i$. It is then straightforward to verify that $\mathcal{D}_i$ indeed is a width 2 GHD for $H'_i$. Thus, as described above, we can use a linear number of calls to \minupdateProb{}($\JJoin{}$) to decide \textsc{3-Sat}. Consequently, if $\ptime \neq \np$, then \minupdateProb($\JJoin{}$) can not be solvable in polynomial time.

\end{document}